\documentclass{article}
\usepackage[numbers]{natbib}  
\usepackage[preprint]{neurips_2025}

\usepackage{times}  % DO NOT CHANGE THIS
\usepackage{helvet}  % DO NOT CHANGE THIS
\usepackage{courier}  % DO NOT CHANGE THIS
\usepackage[hyphens]{url}  % DO NOT CHANGE THIS
\usepackage{graphicx} % DO NOT CHANGE THIS
\usepackage{caption} % DO NOT CHANGE THIS AND DO NOT ADD ANY OPTIONS TO IT
\usepackage{amsmath,amsfonts,bm}

\newcommand{\Acc}{{\operatorname{Acc}}}

\newcommand{\captionc}{{\em (c)}}

\def\eqref#1{equation~\ref{#1}}
\def\1{\bm{1}}

\DeclareMathAlphabet{\mathsfit}{\encodingdefault}{\sfdefault}{m}{sl}
\SetMathAlphabet{\mathsfit}{bold}{\encodingdefault}{\sfdefault}{bx}{n}

\newcommand{\ba}{\boldsymbol{a}}

\newcommand{\bw}{\boldsymbol{w}}
\newcommand{\bx}{\boldsymbol{x}}
\newcommand{\by}{\boldsymbol{y}}

\newcommand{\bM}{\boldsymbol{M}}

\newcommand{\bQ}{\boldsymbol{Q}}

\newcommand{\bW}{\boldsymbol{W}}

\newcommand{\btheta}{\boldsymbol{\theta}}

\newcommand{\cL}{\mathcal{L}}

\usepackage{microtype}
\usepackage{graphicx}
\usepackage{subfigure}
\usepackage{booktabs} % for professional tables
\usepackage{multirow}
\usepackage{makecell}
\usepackage{algorithm}
\usepackage{algorithmic}
\usepackage{wrapfig}
\usepackage{bm}
\usepackage{hyperref}

\usepackage{amsmath}
\usepackage{amssymb}
\usepackage{mathtools}
\usepackage{amsthm}
\usepackage{enumitem}

\usepackage[capitalize,noabbrev]{cleveref}

\DeclareCaptionStyle{ruled}{labelfont=normalfont,labelsep=colon,strut=off} % DO NOT CHANGE THIS

\theoremstyle{plain}
\newtheorem{theorem}{Theorem}[section]

\newtheorem{lemma}[theorem]{Lemma}
\newtheorem{condition}{Condition}
\newtheorem{corollary}[theorem]{Corollary}
\theoremstyle{definition}
\newtheorem{definition}[theorem]{Definition}

\theoremstyle{remark}

\usepackage[textsize=tiny]{todonotes}

\usepackage[utf8]{inputenc} % allow utf-8 input
\usepackage[T1]{fontenc}    % use 8-bit T1 fonts
\usepackage{amsfonts}       % blackboard math symbols
\usepackage{nicefrac}       % compact symbols for 1/2, etc.
\usepackage{xcolor}         % colors

\title{On the Learning Dynamics of Two-layer Linear Networks with Label Noise SGD}

\author{
Tongcheng Zhang$^{1*}$, Zhanpeng Zhou$^{1}$\thanks{Equal contribution.} , Mingze Wang$^{2}$,
Andi Han$^{3,4}$,  \\
\textbf{Wei Huang}$^{4,5\dagger}$, \textbf{Taiji Suzuki}$^{4,6}$, \textbf{Junchi Yan}$^{1}$\thanks{Corresponding authors. Junchi Yan is also with School of AI, Shanghai Jiao Tong University. The authors of Shanghai Jiao Tong University were partly supported by NSFC 92370201.} 
\vspace{4pt}\\
$^1$Sch. of Computer Science and Zhiyuan College, Shanghai Jiao Tong University\\
$^2$Peking University\quad $^3$University of Sydney\\
$^4$RIKEN Center for Advanced Intelligence Project \\
$^5$The Institute of Statistical Mathematics \\
$^6$The University of Tokyo \\
\{a-usually, zzp1012, yanjunchi\}@sjtu.edu.cn,
wei.huang.vr@riken.jp
}
\begin{document}

\maketitle

\begin{abstract}
One crucial factor behind the success of deep learning lies in the implicit bias induced by noise inherent in gradient-based training algorithms.
Motivated by empirical observations that training with noisy labels improves model generalization, we delve into the underlying mechanisms behind stochastic gradient descent (SGD) with label noise.
Focusing on a two-layer over-parameterized linear network, we analyze the learning dynamics of label noise SGD, unveiling a two-phase learning behavior.
In \emph{Phase I}, the magnitudes of model weights progressively diminish, and the model escapes the lazy regime; enters the rich regime.
In \emph{Phase II}, the alignment between model weights and the ground-truth interpolator increases, and the model eventually converges.
Our analysis highlights the critical role of label noise in driving the transition from the lazy to the rich regime and minimally explains its empirical success.
Furthermore, we extend these insights to Sharpness-Aware Minimization (SAM), showing that the principles governing label noise SGD also apply to broader optimization algorithms.
Extensive experiments, conducted under both synthetic and real-world setups, strongly support our theory. Our code is released at https://github.com/a-usually/Label-Noise-SGD.
\end{abstract}

\section{Introduction}
\label{sec:intro}
\vspace{-7pt}
% First paragraph: Understading noise is important. Label noise SGD is a special example worth to analyze.
% The success of deep learning is often attributed not only to architectural innovations but also to the implicit regularization induced by stochasticity in training algorithms.
% Among these stochastic elements, label noise—traditionally viewed as a nuisance—has emerged paradoxically as a critical driver of generalization. 
One central factor behind the success of modern deep learning stems from the implicit bias induced by inherent stochastic noise in gradient-based training algorithms.
While clean training data is ideal, recent studies~\citep{Christopher2019measuring,haochen2021shape,damian2021label} revealed that injecting label noise, or label smoothing during training can paradoxically improve the generalization of neural networks. \cref{fig:label_noise_sgd_generalization} demonstrates the observation where stochastic gradient descent (SGD) with label noise enhances model generalization capability and inherently favors sparser solutions.
\citet{Wang2023ImplicitBO} also demonstrated that training with SGD alone jumps from the local minimum solution with only a small possibility.
These phenomena challenge conventional wisdom and raise a fundamental question: \begin{center}
   \emph{How does label noise, often undesirable in statistical learning, confer benefits in over-parameterized models?}
\end{center} 
Existing theoretical works have tried to understand the mechanisms behind SGD with noisy labels. 
\citet{blanc2020implicit,damian2021label,li2022what} focused on the local geometry around the global minimizers selected by label noise SGD.
\citet{haochen2021shape,vivien2022label} analyzed the behavior of diagonal linear networks under label noise SGD.
However, few attempts to study the learning dynamics of label noise SGD in a more realistic setting. 

\noindent
\textbf{Our Contributions.} In this work, we provide a theoretical analysis of the learning dynamics in a two-layer linear network trained with label noise SGD.

\textbullet~\emph{Theoretical analysis.}
In this work, we rigorously characterize the learning dynamics of a two-layer linear network where both layers are trainable by label noise SGD on a regression task. 
In particular, we identify two phases:
\begin{itemize}[topsep=0.2em,leftmargin=1em]
    \item \textbf{Phase I.} The magnitudes of neuron weights progressively diminish, and the model escapes from the lazy regime~\citep{Chizat2019lazytraining}; enters the rich regime~\citet{geiger2020disentangling}.
    \item \textbf{Phase II.} The neurons increasingly align ground-truth interpolator, and the model becomes sparser.
\end{itemize}

Refer to \cref{tab: two-phase picture} for a summary. In the lazy regime, the dynamics are linear and the model achieves zero training loss with parameters hardly varying~\citep{du2018gradient,li2018learning,du2019gradient,allenzhu2019convergence}.
The lazy regime is often considered undesirable in practice and fails to explain the surprising generalization of neural networks~\citep{Chizat2019lazytraining}.
In contrast, the rich regime, or feature learning, which draws significant attention recently, captures complex non-linear dynamics and is considered beneficial for generalization.
Our analysis highlights the effect of label noise SGD in shifting dynamics from  \emph{lazy} to \emph{rich} regime, serving as a minimalist example to explain its intriguing properties.

Notably, the combination of over-parameterization and the intricate coupling between the first and second layers makes the theoretical analysis of label noise SGD far more challenging than for simpler linear models. 
To the best of our knowledge, our work presents the first detailed theoretical investigation of label noise SGD in networks with two or more trainable layers.

\textbullet~\emph{Extension.} Furthermore, we explore whether the principles underlying label noise SGD apply to broader optimization algorithms.
We extend our findings on label noise SGD to Sharpness-Aware Minimization (SAM)~\citep{foret2021sharpnessaware}. 
We show that SAM also fosters the transition from the \emph{lazy}  to \emph{rich}  regime and promotes sparsity in neural networks.

\emph{In summary}, our work unveils a richer set of implicit bias of label noise SGD. We theoretically analyze the dynamics of SGD with label noise and carefully characterize how it transitions from lazy to rich regime.
Our results offer valuable insights into the mechanisms behind the noise inherent in stochastic learning algorithms.

\begin{figure*}[tb!]
  \begin{center}
    \includegraphics[width=\textwidth]{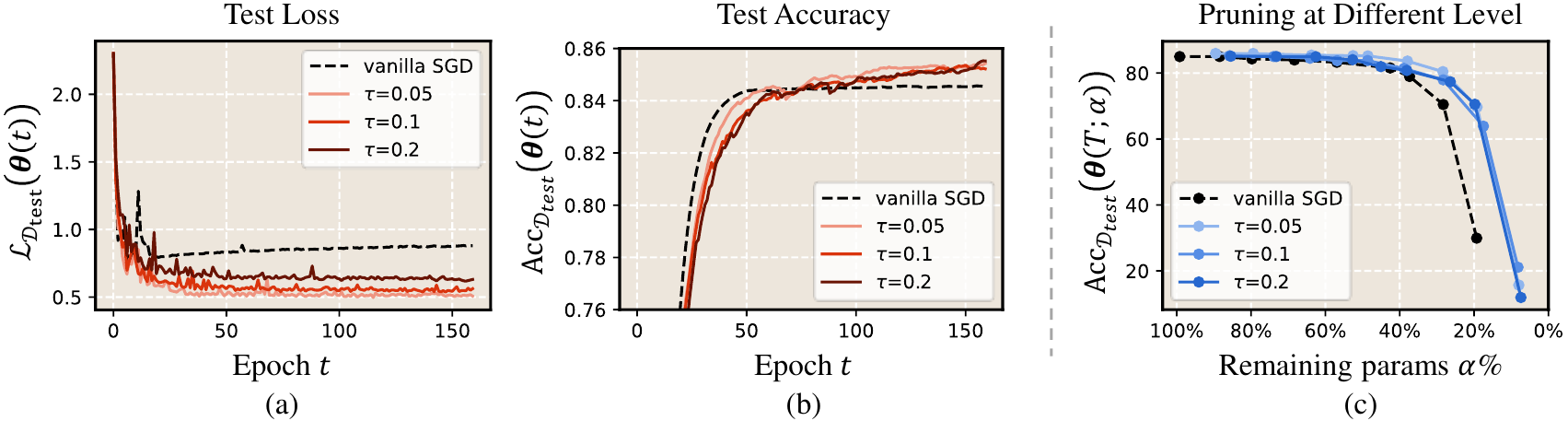}
    % \vspace{-5pt}
    \caption{
        \textbf{(Left).Label noise SGD (\cref{alg:label_noise_sgd}) leads to better generalization.}
        Test loss $\cL(\btheta(t))$ and accuracy $\Acc(\btheta(t))$ vs. training epochs $t$.
         \textbf{(Right). Label noise SGD leads to sparser solutions.}
        Testing accuracy of pruned model $\Acc(\btheta(T; \alpha))$ vs. the percentage of remaining parameters $\alpha$.
        Here, $\btheta(T; \alpha)$ represents the pruned model derived from the pretrained model $\btheta(T)$, with $\alpha$\% of parameters remaining.
        We use both vanilla SGD and label noise SGD to train the models, with no weight decay or momentum. The learning rate is set to $0.1$, and the total number of epochs is $160$.
        % We use both vanilla SGD and label noise SGD to train the models.
        % Label noise SGD outperforms vanilla SGD in both test loss and accuracy across all values of $\tau$.
        Exponential moving average is employed to smooth the test accuracy curves.
        % \textbf{(Right). Label noise SGD leads to sparser solutions.}
        % Testing accuracy of pruned model $\Acc(\btheta(T; \alpha))$ vs. the percentage of remaining parameters $\alpha$.
        % Here, $\btheta(T; \alpha)$ represents the pruned model derived from the pretrained model $\btheta(T)$, with $\alpha$\% of parameters remaining.
        % Similarly, we use both vanilla SGD and label noise SGD to produce the pretrained model $\btheta(T)$.
        % Notably, when training with label noise SGD, pruned model retains a relatively higher performance compared to training with vanilla SGD.
        Results are presented for ResNet-18 \citep{kaiming2016residual} trained on CIFAR-10 \citep{krizhevsky2009learning}, across different label noise probabilities, $(\tau \in \{0.05, 0.1, 0.2\})$. 
        As shown in figure (a) and figure (b), label noise SGD consistently outperforms vanilla SGD in both test loss and accuracy across different values of the label flipping probability $\tau$, providing an around $1.5$\% improvement in test accuracy.
        As shown in \cref{fig:label_noise_sgd_generalization}~\captionc, models trained with label noise SGD maintain higher performance at the same sparsity level compared to those trained with vanilla SGD.}
        % Detailed settings are in \cref{suppl:exp_settings}. 
    \label{fig:label_noise_sgd_generalization}
    \vspace{-5pt}
  \end{center}
\end{figure*}

\section{Related Work}
\label{sec:related_work}
\vspace{-7pt}

{\bf Lazy Regime.}
Numerous theoretical studies investigated the learning dynamics of highly over-parameterized neural networks in the {\em lazy} (or kernel) regime~\citep{jacot2018NTK,du2018gradient,du2019gradient,allenzhu2019convergence,zou2020gradient}. In this regime, the model behaves as its linearized model around initialization throughout training, making it equivalent to a deterministic kernel, specifically the neural tangent kernel (NTK)~\citep{jacot2018NTK}. The lazy regime typically occurs in over-parameterized models with {\em relatively large initialization}~\citep{Chizat2019lazytraining}. While global exponential convergence can be established in this setting, the lazy dynamics fail to explain the generalization advantage over kernel methods—a fundamental question in understanding the success of deep learning. 

\noindent
{\bf Rich Regime.}
In contrast to the lazy regime, where learning dynamics remain linear, the {\em rich} regime\footnote{This term broadly refers to learning behaviors that deviate from the lazy regime.}, also known as feature learning regime, exhibits complex nonlinear dynamics~\citep{chizat2018global,mei2019mean}, including the initial alignment phenomenon~\citep{maennel2018gradient,luo2021phase} and saddle-to-saddle dynamics~\citep{jacot2021saddle,abbe2023sgd,pesme2023saddle}.
Some studies have demonstrated that the initialization scale governs the emergence of the rich regime in (S)GD, which typically occurs at {\em small initialization scales}~\citep{geiger2020disentangling,Woodworth2020kernel}. In this regime, it is shown that small initialization induces simplicity biases, leading to sparse or low-rank features~\citep{maennel2018gradient,li2020towards,lyu2021gradient,boursier2022gradient,wang2023understanding,min2023early}.
Subsequent work further revealed that the relative scale of initializations~\citep{azulay2021shape,kunin2024get} and their effective rank~\citep{liu2024how} can similarly induce feature learning.
Beyond initializations, factors like weight decay~\citep{lewkowycz2020on,jacot2022feature,lyu2023dichotomy} and large learning rates~\citep{lewkowycz2020large,ba2022highdimensional} have also been shown to drive the rich regimes.

\noindent
\textbf{Label Noise SGD Theories.}
Many existing theoretical works have analyzed label noise SGD from the perspective of implicit regularization. 
\citet{blanc2020implicit,damian2021label,li2022what} showed that label noise implicitly regularizes the sharpness of the minimizers. 
\citet{haochen2021shape,vivien2022label} proved that training with label noise helps recover the sparse ground-truth interpolator in a diagonal linear network setup.
\citet{takakura2024meanfield} analyzed the implicit regularization of label noise from a kernel perspective. 
In addition to implicit regularization, \citet{huh2024generalization} derived a generalization bound for label noise SGD. \citet{Han2025OnTR} theoretically analyzed the training dynamics with static label noise. \citet{Huang2025HowDL} analyzed the generalization capability in low SNR regimes. \citet{NEURIPS2024_6ec81faa} showed that label noise SGD tends to gradually reduce the rank of the parameter matrix.

% our work bridges empirical observations and theoretical analysis. 
% Empirically, we show that label noise SGD leads to improved generalization and sparsity. 
\emph{In comparison},  we theoretically analyze the learning dynamics of label noise SGD in an over-parameterized two-layer linear network, highlighting the transition from lazy to rich regime. Analyzing the two-layer linear network with label noise SGD requires careful treatment of the update rule of both layers, which introduces complex coupling effect between the first and the second layer parameters, thus posing significant challenges to theoretical analysis.  

\section{Preliminaries}
\label{sec:prelim}
\vspace{-7pt}

\begin{algorithm}[tb!]
   \caption{Label Noise SGD~\citep{haochen2021shape,damian2021label}}
   \label{alg:label_noise_sgd}
\textbf{Input}: Initial parameters $\boldsymbol{\theta}(0)$, step size $\eta$, label flipping probability $\tau$, batch size $B$, steps $T$.
\begin{algorithmic}[1]
\FOR{$t=0$ {\bfseries to} $T-1$}
    \STATE Sample a batch $\mathcal{B}_t \in [n]^B$ uniformly.
    \FOR{each $i\in \mathcal{B}_t$}
        \STATE Sample $u \sim {\rm Uniform}(0, 1)$.
        \IF{$u < \tau$}
            \STATE $\tilde{y}_i = $ sample from $[c] \backslash \{y_i\}$.
        \ELSE
            \STATE $\tilde{y}_i = y_i$.
        \ENDIF
        \STATE $\hat{\ell}_i(\boldsymbol{\theta}(t)) = \ell(f(\boldsymbol{\theta}(t); \boldsymbol{x}_i), \tilde{y}_i)$.
    \ENDFOR
    \STATE $\hat{\mathcal{L}}(\boldsymbol{\theta}(t)) = \frac{1}{B}\sum_{i \in \mathcal{B}_t} \hat{\ell}_i(\boldsymbol{\theta}(t))$.
    \STATE $\boldsymbol{\theta}(t+1) = \boldsymbol{\theta}(t) - \eta \nabla \hat{\mathcal{L}}(\boldsymbol{\theta}(t))$.
\ENDFOR
\end{algorithmic}
\end{algorithm}

\noindent
\textbf{Basic Notation and Setup.}
Denote $[k]=\{1,2,\ldots,k\}$.
Let $\mathcal{D} = \{(\boldsymbol{x}_i, y_i)\}_{i=1}^n$ be the training set, where $\boldsymbol{x}_i\in \mathbb{R}^{d}$ is the input and $y_i \in \mathbb{R}$ is the label/target of the $i$-th data point.
Let $f: \mathcal{D} \times \mathbb{R}^d \to \mathbb{R}$ be the model function and let $f(\boldsymbol{x}_i; \boldsymbol{\theta})$ be the model output on the $i$-th data point, where $\boldsymbol{\theta} \in \mathbb{R}^p$ are the model parameters. 
The loss of the model at the $i$-th sample $(\boldsymbol{x}_i, y_i)$ is denoted as $\ell(f(\boldsymbol{x}_i; \boldsymbol{\theta}), y_i)$, simplified to $\ell_i(\boldsymbol{\theta})$.
The loss over the training set is then given by $\mathcal{L}_{\mathcal{D}}(\boldsymbol{\theta}) = \frac{1}{n} \sum_{i=1}^n \ell_i(\boldsymbol{\theta})$.
Note that we consider classification tasks in our empirical observations, where $y_i \in [c]$ and $c$ are the number of classes.
We also use $\text{Acc}_{\mathcal{D}} (\boldsymbol{\theta})$ to denote the classification accuracy of $f(\boldsymbol{\theta})$ on the dataset $\mathcal{D}$. 

Throughout the paper, bold lowercase letters denote vectors, and bold uppercase letters represent matrices. 
The unbolded lowercase letters with subscripts indicate the entries of vectors or matrices, such as $x_i$ for the $i$-th entry of $\boldsymbol{x}$ and $a_{i,j}$ for the $(i, j)$-th entry of $\boldsymbol{A}$. 
For simplicity, we use $\Vert \boldsymbol{x}\Vert$ as $\Vert \boldsymbol{x}\Vert_2$ for a vector $\boldsymbol{x}$ and $\Vert \boldsymbol{X}\Vert$ as $\Vert \boldsymbol{X}\Vert_F$ for a matrix $\boldsymbol{X}$.

\noindent
\textbf{Label Noise SGD.}
We recall the algorithm of label noise SGD in \cref{alg:label_noise_sgd}. We focus on a classification setting, where the label flipping probability $\tau$ governs the noise level; for simplicity, our theoretical analysis  considers a regression task. 
Specifically, label noise SGD can be adapted to regression by replacing $\tau$ with the noise variance $\sigma^2$. 
In this context, the noisy label $\tilde{y}_i$ is generated by $
    \tilde{y}_i = y_i + \epsilon, \text{where } \epsilon \sim \{-\sigma, \sigma\}.$
Assuming the squared loss without loss of generality, the training loss at the $i$-th data is given by: \begin{equation}
    \hat{\ell}_i(\boldsymbol{\theta}(t)) =  \frac{1}{2} \left\lvert f(\boldsymbol{\theta}(t); \boldsymbol{x}_i)- y_i - \epsilon \right\rvert^2.\label{eq:squared_loss}
\end{equation}
This setup is widely adopted in recent theoretical advances on label noise SGD~\citep{damian2021label,haochen2021shape,li2022what,vivien2022label,eun2024generalization}.

\section{Theoretical Analysis: The Learning Dynamics of Label Noise SGD}
\label{sec:theory}
\vspace{-7pt}

This section presents a theoretical analysis of learning dynamics in a two-layer linear network, characterizing the phase transition from lazy to rich regimes under label noise SGD.

% This section aims to explain the transition from lazy to rich regime of label noise SGD by analyzing its learning dynamics on a two-layer linear network.
\noindent
\textbf{Roadmap.}
In \cref{sec:setup}, we formulate the problem setup and provide an overview of our theoretical results.
In \cref{sec:phase_I}, we present formal theorems.
In \cref{sec:simulate}, we demonstrate the role of oscillation using a simple scenario.
In \cref{sec:phase_II}, we validate our theory by extensive experiments under both synthetic and real-world setups.
 
% Specifically, we unveil a multi-stage learning dynamics of label noise SGD.
\subsection{Setup and Overview: A Two-Layer Linear Network}
\label{sec:setup}

% We study the learning dynamics of a two-layer shallow linear network, which are mathematically tractable when trained with label noise SGD.

\textbf{Problem Setup.}
We consider a regression task where each data pair $(\boldsymbol{x}_i, y_i) \in \mathcal{D}$ maps input $\boldsymbol{x}_i \in \mathbb{R}^d$ to its corresponding target $y_i \in \mathbb{R}$.
We solve this task using a two-layer linear network of the form: \begin{align}
    \hat{y}_i = \boldsymbol{a}^{\top} \boldsymbol{W} \boldsymbol{x}_i,
\end{align}
where $\boldsymbol{W} \in \mathbb{R}^{m\times d}$ and $\boldsymbol{a} \in \mathbb{R}^{m}$.
Here, $m$ represents the number of neurons and $\boldsymbol{w}_i$ denotes the $i$-th neuron of $\boldsymbol{W}$.

\textbullet~\emph{Label noise SGD.} The network’s parameters $\boldsymbol{\theta}=\boldsymbol{a}^{\top} \boldsymbol{W}$ are optimized using label noise SGD with a squared loss function (see \cref{eq:squared_loss}).
The update rule is written as: \begin{align}
    &\boldsymbol{\theta}(t+1) = \boldsymbol{\theta}(t) - \eta \nabla_{\boldsymbol{\theta}} \hat\ell_{\xi_t}(\boldsymbol{\theta}(t)), \label{passage:theta update}\\
    &\hat\ell_{\xi_t}(\boldsymbol{\theta}(t)) = \frac{1}{2} \left\lvert f(\boldsymbol{\theta}(t); \boldsymbol{x}_{\xi_t}) - y_{\xi_t} - \epsilon_t \right\rvert^2,
\end{align}
where $\xi_t \in [n]$ represents the index of a randomly sampled training sample at iteration $t$, and the noise $\epsilon_t \sim \{-\sigma, \sigma\}$ is controlled by the variance $\sigma^2$. 

\textbullet~\emph{Initialization.} We consider label noise SGD starting from the following initializations: for $i\in [m]$ and $j\in [d]$, \begin{align}
    w_{i, j}(0) \stackrel{\rm i.i.d.}{\sim} \frac{1}{\sqrt{d}}\mathcal{N}(0, I) {\rm\ and\ } a_i(0) \stackrel{\rm i.i.d.}{\sim}\frac{1}{\sqrt{m}} \mathcal{N}(0, I).\label{eq:initialization}
\end{align}
This initialization scheme is commonly referred to as the NTK initialization~\citet{jacot2018NTK}. 
~\citet{allenzhu2019convergence} showed that training over-parameterized models initialized as \cref{eq:initialization} with SGD stays in the lazy regime.

\textbullet~\emph{Data generation.} Without loss of generality, we assume each input $\boldsymbol{x}_i$ is drawn from $\mathcal{N}(0, \boldsymbol{I}_{d\times d})$, and that there exists at least one interpolating parameter $\boldsymbol{\theta}^{\star}$ that perfectly fits the training set, i.e., $\mathcal{L}_{\mathcal{D}}(\boldsymbol{\theta}^{\star})=0$\footnote{$\mathcal{L}_{\mathcal{D}}(\boldsymbol{\theta}^{\star}) = \frac{1}{n} \sum_{i=1}^n \ell_i(\boldsymbol{\theta}^{\star}) = \frac{1}{n} \sum_{i=1}^n \frac{1}{2} \lvert f(\boldsymbol{\theta}^{\star}; \boldsymbol{x}_i), y_i \rvert^2$}.

\begin{table*}[!tb]
    \centering    
    \begin{tabular}{l|l|l}
    \hline\hline
        \multirow{2}*{\makecell[l]{\bf I}}
        & {\em The magnitudes of model weights progressively diminish;} 
        % as the stability condition~\citep{wu2022alignment} required for SAM is not satisfied.  
        & \bf \cref{passage norm decay each step} \\
        & {\em The model escapes the lazy regime; enters the rich regime.}
        % due to the subquadratic property of the landscape~\citep{ma2022beyond}. 
        & \bf \cref{passage: Phase I main theorem} \\
        \hline
        \multirow{2}*{\makecell[l]{\bf II}} 
        & {\em The alignment between model weights and the ground-truth } & \bf \cref{passage: phase II alignment} \\
        & {\em  interpolator increases; the model eventually converges.} & \bf \cref{passage: phase II norm increase} \\\hline\hline
    \end{tabular}
    \caption{Overview of the two-phase picture and corresponding theoretical results.}
    \label{tab: two-phase picture}
\end{table*}

\noindent
\textbf{Overview.} 
We first state our main conditions. 
\begin{condition}\label{cond:main}
Suppose there exists a sufficiently large constant~\footnote{$C\geq \max(e^{-\frac{(\sigma  C_{data})^2}{3}},(\frac{(1-3/(4\sqrt{\pi}))\cdot2\sqrt{d}}{1/2-3/(4\sqrt{\pi})}{\sqrt{\pi}})^8, e C_{data}^2)$, where $C_{data}$ is a constant defined in 
\cref{cond:main}\ A(5)} $C$  such that the following holds:
\begin{enumerate}[topsep=0em,leftmargin=0.15in]
    \item {\rm \bf (A1) Model width.} The width of the network $m$ satisfies
            $m = \Omega\left( \frac{1}{\sqrt{\eta}}\right)$.
    \item {\rm \bf (A2) Learning rate.} The learning rate satisfies
    $\eta \leq \frac{1}{C^{96}}$.
    \item {\rm \bf (A3) Dataset size.} The training set size satisfies $n\geq \frac{1}{\eta^2}$. 
    \item {\rm \bf (A4) Scale of the optimal parameter.} The ground-truth interpolator satisfies $\|\boldsymbol{\theta}^{\star}\|\leq m^{-1/4}$.
    \item  {\rm \bf (A5) Input magnitude.} The maximum norm of the input samples satisfies $\max_{i}\|\boldsymbol{x}_i\|\leq C_{data}$. \label{Input magnitude}
    \item  {\rm \bf (A6) Dimension of sample.} The dimension of a single sample $d$ satisfies $d\geq \frac{9(\ln2)\cdot K^4}{2c}$.
\end{enumerate}
\end{condition} 
\noindent
\textbf{Remarks on \cref{cond:main}.}
% We first require that the maximum norm of the input across training samples is upper-bounded.
% This condition is well-defined, as the training set is finite.
Specifically, A1 ensures over-parameterization, where $m \gg d$, which is important for enabling the progressively diminishing of weight norm in Phase I. 
A2 ensures the step size is small enough to allow Phase I to persist over a long range of iterations. 
A3 addresses the dataset size: when using gradient descent, a substantial volume of training data is typically required.
A4 assumes the sparse ground-truth interpolator. 
For the sake of clarity in our theoretical analysis, we additionally introduce the A5 to give an upper bound of input norm and A6 to give a lower bound of input dimension. 
In A5, $C_{data}$ is a constant and by definition of $C$ in A1, we have $C_{data}\leq C \ll m$. In A6, both $K$ and $c$ are constant and defined in Appendix (Lemma A.5). 

Under these conditions, we can state our main results: we identify a two-phase picture in the training dynamics of label noise SGD, as outlined in \cref{tab: two-phase picture}. 
In the following, we will formally present our formal theorems, with each claim supported by corresponding analysis.

\begin{figure*}[tb!]
  \begin{center}
    \includegraphics[width=\textwidth]{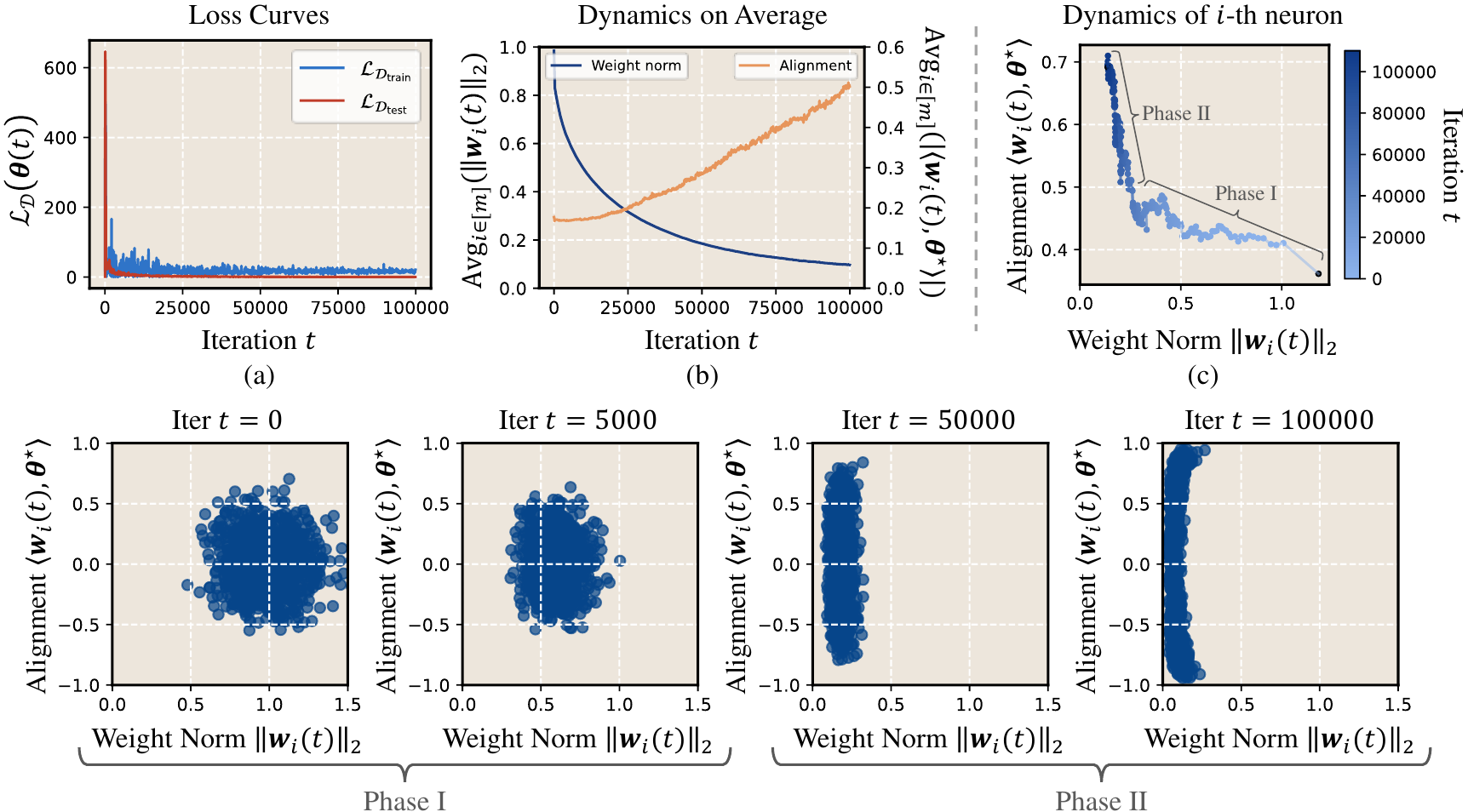}
     %\vspace{-5pt}
    \caption{
    Two-phase dynamics of label noise SGD under synthetic setup.
        We replicate the synthetic problem setup from \cref{sec:exp_validate}.
        (a) Loss curves.
        Training $\mathcal{L}_{\mathcal{D}_{\rm train}}(\boldsymbol{\theta}(t))$ and test loss $\mathcal{L}_{\mathcal{D}_{\rm test}}(\boldsymbol{\theta}(t))$ vs. training iteration $t$.
        (b) Learning dynamics on average.
        The mean neuron norm $\text{Avg}_{i\in [m]}(\Vert \boldsymbol{w}_i (t) \Vert_2)$ and the mean neuron alignment $\text{Avg}_{i\in [m]}(\langle \boldsymbol{w}_i(t), \boldsymbol{\theta}^{\star} \rangle)$ vs. training iteration $t$.
        (c) Learning dynamics of $i$-th neuron.
        The alignment of the $i$-th neuron $\langle \boldsymbol{w}_i(t), \boldsymbol{\theta}^{\star} \rangle$ vs. its weight norm $\Vert \boldsymbol{w}_i (t) \Vert_2$, with darker points indicating larger $t$.
        (Bottom) Complete view of dynamics of each neuron.
        This plot is similar to (c); yet instead of focusing on a single neuron, we plot the status over iterations.
    }
    \label{fig:labelnoisesgd_synthetic}
    %\vspace{-10pt}
  \end{center}
\end{figure*}
\subsection{Phase I: Progressively Diminishing; From the Lazy to the Rich Regime}
\label{sec:phase_I}
In this section, we present theoretical support for Phase I. Inspired by \citet{du2018gradient,du2019gradient,allenzhu2019convergence}, we first introduce the definition of the lazy regime. 
\begin{definition}[The lazy regime]\label{def:lazy_regime}
    $\forall i\in [m]$, it holds that $\|\boldsymbol{w}_i(t)-\boldsymbol{w}_i(0)\|\leq\frac{1}{\sqrt{m}}$.
\end{definition}
\cref{def:lazy_regime} depicts a minimal variation of model weights from its initialization at time $t$.
Based on \cref{def:lazy_regime}, we establish the following theorem.
\begin{theorem}[Escaping the lazy regime]
    Suppose \cref{cond:main}~(A1-2, 4-6) hold and consider the update rule in \cref{passage:theta update}.
    With probability at least $1-O(\frac{1}{m})$, all the neurons $\boldsymbol{w}_{i}$ ($i\in[m]$) escape from the lazy regime at time $T_1 = \frac{384\sqrt{\log m}}{\sigma^2\eta^2\sqrt{m}}$. \label{passage: Phase I main theorem}
\end{theorem} 
\noindent
\textbf{Insights from \cref{passage: Phase I main theorem}.}
\cref{passage: Phase I main theorem} indicates that in Phase I, label noise SGD facilitates the transition from the lazy to the rich regime.
Indeed, such transition is induced by the \emph{progressively diminishing} of the first-layer weights $\boldsymbol{W}$.
Specifically, for each neuron $\boldsymbol{w}_i (i\in [m])$ at time $T$, we can easily derive that $\Vert \boldsymbol{w}_{i}(T)\Vert^2 =\Vert \boldsymbol{w}_{i}(0) \Vert^2 + \eta^2\sum_{j = 0}^{T-1}\Delta W_i(j) -a_{i}(0)^2+a_{i}(T)^2$ and $\Delta W_i(j) = -\nabla \hat\ell_{\xi_j}(\boldsymbol{\theta}(j))^2((\boldsymbol{x}_{\xi_j}^\top \cdot \boldsymbol{w}_{i}(j))^2 -  a_{i}(j)^2\cdot\Vert \boldsymbol{x}_{\xi_j}\Vert^2)$.
Since $\boldsymbol{a}(0)$ is initialized small, the term $\nabla \hat\ell_{\xi_j}(\boldsymbol{\theta}(j))^2(\boldsymbol{x}_{\xi_j}^\top \cdot \boldsymbol{w}_{i}(j))^2$ dominates the evolution of the weight norm.
Notably, by \cref{passage:theta update}, we have \begin{equation*}
    \nabla \hat\ell_{\xi_j}(\boldsymbol{\theta}(j))^2(\boldsymbol{x}_{\xi_j}^\top \cdot \boldsymbol{w}_{i}(j))^2 = (a_i(j+1)-a_i(j))^2.
\end{equation*}
Consequently, the evolution of the first-layer weight norm is primarily determined by the oscillations of the neurons in the second layer. 
Intuitively, label noise accelerates the oscillations in the second layer, thereby contributing to the progressive diminishing of the first-layer weights.

\noindent
\textbf{Proof Sketch of \cref{passage: Phase I main theorem}.}
The proof relies on \cref{passage norm decay each step}. 

\begin{lemma}[Progressively diminishing at each step]\label{passage norm decay each step}
    Suppose \cref{cond:main}~(A1-2, 4-6) hold and consider the update rule in \cref{passage:theta update}. Assume the model is still under the lazy regime at step $T$, then with probability at least $1-O(\frac{1}{m})$, for all the iterative steps $j\leq T_1$ and for every $i\in[m]$: 
    \begin{enumerate}[topsep=0em,leftmargin=0.15in]
        \item $\Delta W_i(j) \leq 0$ with probability at least $1 - \frac{\rho}{m^{1/8}}$. Furthermore,  $\Delta W_i(j) \leq -(\frac{\sigma}{4})^2$ with probability at least $ \frac{1}{4}$.
        \item $\Delta W_i(j) > 0$ with probability at most $\frac{\rho}{m^{1/8}}$.  Furthermore, we have $\Delta W_i(j)\leq O(1)$.
    \end{enumerate}
    where $\rho = \frac{2\sqrt{d}}{\sqrt{\pi}}$ is a constant.
\end{lemma}
\cref{passage norm decay each step} states that with high probability, $\Delta W_i(j)$ is negative at each step. 
Furthermore, with probability at least $\frac{1}{4}$, $\Delta W_i(j)$ significantly deviates from zero. In our setup, we can prove that with high probability, $|a_i(t)|$ remains small throughout training. 
Thus, $\Delta W_i(j)$ primarily governs the change in $\Vert \boldsymbol{w}_i\Vert^2$ at each step. 
With \cref{passage norm decay each step}, we can demonstrate the progressive diminishing of the first-layer weights $\boldsymbol{W}$, thereby driving the transition from lazy to rich regime.

\subsection{Simulation Setup: Oscillation Induces Progressive Diminishing}
\label{sec:simulate}
In the previous analysis, we have shown that the oscillation of the second layer plays a central role in the progressive diminishing of the first-layer weights $\boldsymbol{W}$, and label noise SGD facilitates the oscillations, leading to the phase transition.

\noindent
\textbf{Simulation Setup.}
Inspired by the discovery, we propose to simulate the oscillation of the second layer via a simple three-state Markov process. To eliminate the interference of SGD noise, the first-layer weights $\boldsymbol{w}_i$ follow GD update rule:
\begin{gather}
    \label{passage: iter_w1_eta}
    \boldsymbol{w}_{i}(t+1) = \boldsymbol{w}_{i}(t) - \eta\cdot a_{i}(t)\nabla\mathcal{L}_{\mathcal{D}}(\boldsymbol{\theta}(t))\boldsymbol{x}_{i}, \\
    \label{passage: iter_w2_eta}
    a_{i}(t+1) = a_{i}(t)  + \delta_i(t),
\end{gather}
where 
\begin{equation*}
    \delta_i(t)=\left\{
    \begin{aligned}
    & -\eta^{0.25} \quad \text{if $a_{i}(t)= \eta^{0.25}$}\\
    & \eta^{0.25} \quad\quad \text{if $a_{i}(t)= -\eta^{0.25}$} \\
    & \sim \{-\eta^{0.25}, \eta^{0.25}\}
    \ \ \ \ \text{if $a_{i}(t)= 0$}
    \end{aligned}
\right.
\end{equation*}
The initialization of $a_i$ is set to $\eta^{0.25}$ or $-\eta^{0.25}$, each with probability $1/4$ and set to $0$ with probability $1/2$. 
The initialization of $\boldsymbol{w}_i$ remains consistent with \cref{eq:initialization}. With this design, the neurons in the second layer exhibit strong oscillations within a small range. 
The following lemma demonstrates the progressive diminishing under this algorithm.
\begin{lemma}[Progressively diminishing under simulation setup]\label{passage algorithm 2 decay}
     Suppose \cref{cond:main}~(A1-3, 5-6) holds (let $m = \frac{1}{\sqrt{\eta}}$) and consider the update rule in \cref{passage: iter_w1_eta,passage: iter_w2_eta}, there exists a step $t_0\leq \frac{1}{\eta^2}$ such that~\footnote{We simply denote $\mathbb{E}_{ \{x^{(t-1)}\}_{i=0}^{t-1} \in D, \{\epsilon_i\}_{i=0}^{t-1}\sim \{-\sigma,+\sigma\}}$ as $\mathbb{E}$}
     \begin{align}
        \mathbb{E}[\frac{1}{m}\sum_{i=1}^m \Vert \boldsymbol{w}_{i}(t_0)\Vert^2] \leq\sqrt{\eta}.
    \end{align}
\end{lemma}
\noindent
\textbf{Insights into \cref{passage algorithm 2 decay}.}
\cref{passage algorithm 2 decay} further confirms our key message that label noise SGD primarily contributes to the oscillation of the second layer, which induces the progressively diminishing phenomenon, ultimately leading to the transition from the lazy to the rich regime.
\begin{figure*}[tb!]
  \begin{center}
    \includegraphics[width=\textwidth]{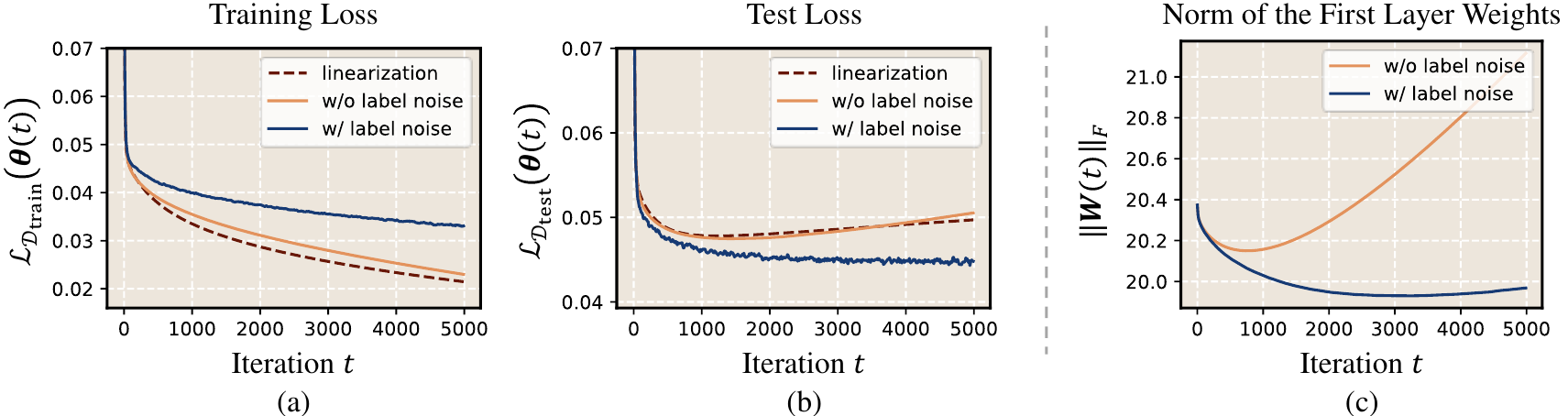}
    % \vspace{-5pt}
    \caption{
        Label noise SGD induces the rich regime.
        (a, b).
        Training $\mathcal{L}_{\mathcal{D}_{\rm train}}(\boldsymbol{\theta}(t))$ and test loss $\mathcal{L}_{\mathcal{D}_{\rm test}}(\boldsymbol{\theta}(t))$ vs. training epochs $t$.
        % Models are optimized using GD, full-batch SAM,
        Label noise SGD induces the progressively diminishing phenomenon.
        (c).
        The first-layer weight norm $\Vert \boldsymbol{W}(t) \Vert_F$ vs. training epochs $t$.
        We use GD to train the models with NTK parameterization~\citep{jacot2018NTK}, both with and without label noise.
        We also train a linearized model with GD as baseline.
        % SAM deviate from both GD and linearized model trajectories, demonstrating escape from the lazy regime.
        Results are presented for WideResNets trained on a random subset of 64 images from CIFAR-10 due to the $O(n^2)$ computational complexity of NTK.
        % More detailed settings are in \cref{suppl:exp_settings}. 
    }
    \label{fig:labelnoisesgd_real}
    % \vspace{-15pt}
  \end{center}
\end{figure*}
\subsection{Phase II: Alignment and Convergence}
\label{sec:phase_II}
In this section, we present theoretical support for Phase II.
When all the neurons satisfy $\|\boldsymbol{w}_i\|,|a_i|\leq\sqrt{\eta}$, we say that Phase II begins. 
This situation is analogous to small initialization~\citep{geiger2020disentangling,Woodworth2020kernel}. 
During this phase, the neurons in the first layer rapidly align with the ground-truth interpolator $\boldsymbol{\theta}^{\star}$. 

Notice that we consider gradient descent in Phase II for simplicity. 
This simplification maintains mathematical tractability without affecting our conclusion in Phase II.
The following lemmas formalize our results in Phase II.
\begin{lemma}[Alignment]\label{passage: phase II alignment}
     Suppose \cref{cond:main}~(A1-3,5-6) holds and consider gradient descent for updates.
     Assume that phase II begins at time $t_1$, then at time $t_2=t_1+T_2$, $T_2=\frac{1}{\|\boldsymbol{\theta}^{\star}\|}\cdot \ln(\frac{1}{\eta})$, for any neuron $\boldsymbol{w}_{i}$ it holds
    \begin{equation}
        \dfrac{\left\lvert \langle\boldsymbol{\theta}^{\star},\boldsymbol{w}_{i}(t_2)\rangle\right\lvert}{\|\boldsymbol{\theta}^{\star}\|\cdot\|\boldsymbol{w}_{i}(t_2)\|}\geq  1- \left\lvert O(\ln\frac{1}{\eta}\cdot\sqrt{\eta})\right\lvert.
    \end{equation} 
\end{lemma}
\begin{lemma}[Convergence]\label{passage: phase II norm increase}
    Suppose \cref{cond:main}~(A1-3, 5-6) holds and consider gradient descent for updates.
    Assume all the neurons are perfectly aligned at step $t_2$. Let $t_3 = t_2 + \frac{1}{\|\boldsymbol{\theta}^{\star}\|^2}\cdot\frac{\ln(1/\eta)}{\eta} $. Using gradient descent, we have $\|\boldsymbol{\theta}(t_3) - \boldsymbol{\theta}^{\star}\|\leq |O(\eta\cdot\ln\frac{1}{\eta})|$. Furthermore, for any neuron $\|\boldsymbol{w}_i(t_3)\|\geq \sqrt{\eta}$ ($i\in[m]$), we have 
    \begin{equation}
\dfrac{\left\lvert\langle\boldsymbol{\theta}^{\star},\boldsymbol{w}_i(t_3)\rangle\right\lvert}{\|\boldsymbol{\theta}^{\star}\|\cdot\|\boldsymbol{w}_{i}(t_3)\|}\geq 1 - \left\lvert O(\eta\cdot\ln\frac{1}{\eta})\right\lvert.
    \end{equation}
\end{lemma}
\noindent
\textbf{Insights from \cref{passage: phase II alignment,passage: phase II norm increase}.} 
\cref{passage: phase II alignment} indicates that the directions of each neuron rapidly align to a common direction, that of the ground-truth interpolator $\boldsymbol{\theta}^{\star}$.
This alignment process is critical in Phase II, where the optimization shifts from the progressive diminishing phase to a more stable and efficient convergence towards the global minimum. 
Once perfect alignment is achieved, \cref{passage: phase II norm increase} guarantees that after $T_3 = O(\frac{-\ln\eta}{\eta})$ steps, $\boldsymbol{\theta}(t)$ converges to the  solution $\boldsymbol{\theta}^{\star}$. 

\subsection{Experiments: Synthetic and Real-World Setups}\label{sec:exp_validate}
In this section, we provide extensive empirical evidence support for our theory.

%\vspace{1mm}
\noindent
\textbf{The Two-phase picture under synthetic setups.}
The synthetic experiments precisely replicate the problem setup in \cref{sec:setup}. In \cref{fig:labelnoisesgd_synthetic} \textit{(b)}, the averaged neuron norm $\frac{1}{m}\sum_{i=1}^m \Vert \boldsymbol{w}_i(t) \Vert$ initially drops as $t$ increases, suggesting the progressive diminishing phenomenon in Phase I.
Afterwards, the averaged neuron alignment $\frac{1}{m}\sum_{i=1}^m \langle \boldsymbol{w}_i(t), \boldsymbol{\theta}^{\star}\rangle$ rapidly increases, implying the convergence to the global solution in Phase II.
Additionally, in \cref{fig:labelnoisesgd_synthetic} \textit{(c)} and \textit{(bottom)}, we visualize the dynamics of each neuron in the training process, where a clear two-phase pattern is observed.

%\vspace{1mm}
\noindent
\textbf{The transition from the lazy to rich regime under real-world setups.}
The real-world experiments are presented for WideResNets~\citep{zagoruyko2016wide} trained on a small subset of CIFAR-10.
Specifically, we compare the loss curves of models trained with and without label noise. 
We also train a linearized model without label noise as a baseline.
In \cref{fig:labelnoisesgd_real} \textit{(a)} and \textit{(b)}, the model trained without label noise behaves similarly to its linearized counterparts, indicating the lazy regime; whereas the model trained with label noise follows a distinctly different training trajectory, suggesting the rich regime. Additionally, we also plot the evolution process of the first-layer weight norm $\Vert \boldsymbol{W} (t) \Vert$. 
In \cref{fig:labelnoisesgd_real} \textit{(c)}, when training with label noise, the first-layer weight norm notably decreases, especially compared to the case without label noise
This result further validates our progressively diminishing phenomenon in real-world setups.

\begin{figure*}[tb!]
  \begin{center}
    \includegraphics[width=\textwidth]{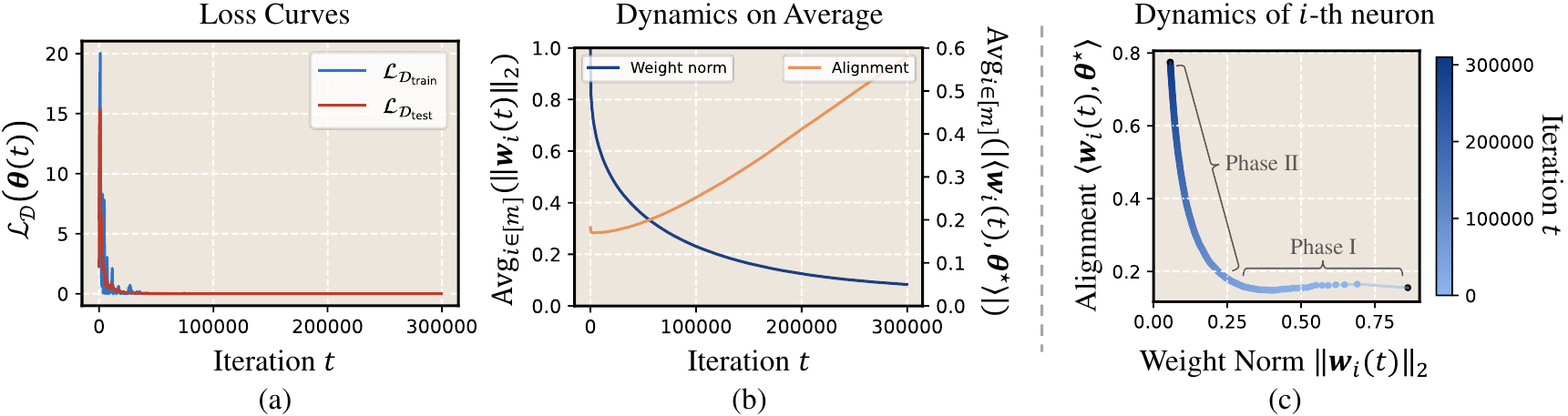}
    % \vspace{-5pt}
    \caption{
        Two-phase dynamics of SAM under synthetic setup.
        We replicate the synthetic problem setup from \cref{sec:exp_validate}, replacing label noise SGD with SAM.
        (a) Loss curves.
        Training $\mathcal{L}_{\mathcal{D}_{\rm train}}(\boldsymbol{\theta}(t))$ and test loss $\mathcal{L}_{\mathcal{D}_{\rm test}}(\boldsymbol{\theta}(t))$ vs. training iteration $t$.
        (b) Learning dynamics on average.
        The averaged neuron norm $\text{Avg}_{i\in [m]}(\Vert \boldsymbol{w}_i (t) \Vert_2)$ and the averaged neuron alignment $\text{Avg}_{i\in [m]}(\langle \boldsymbol{w}_i(t), \boldsymbol{\theta}^{\star} \rangle)$ vs. training iteration $t$.
        (c) Learning dynamics of $i$-th neuron.
        The alignment of $i$-th neuron $\langle \boldsymbol{w}_i(t), \boldsymbol{\theta}^{\star} \rangle$ vs. its weight norm $\Vert \boldsymbol{w}_i (t) \Vert_2$, with darker points indicating larger iteration $t$.
        % More detailed settings are in \cref{suppl:exp_settings}. 
    }
    \label{fig:sam_synthetic}
    \vspace{-5pt}
  \end{center}
\end{figure*}

\section{Extension: From Label Noise SGD to SAM}
\vspace{-7pt}
We conjecture that the principles governing label noise SGD also apply to broader optimization algorithms.
In this section, we extend our findings from label noise SGD to SAM.

\begin{figure*}[tb!]
  \begin{center}
    \includegraphics[width=\textwidth]{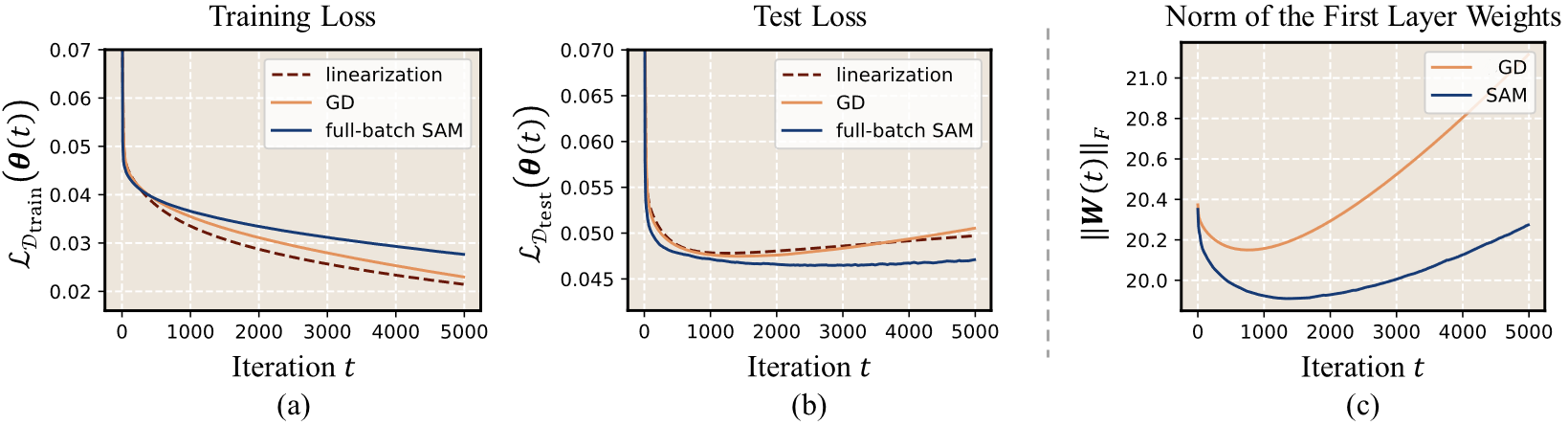}
      %\vspace{-5pt}
    \caption{
        SAM induces the rich regime.
        Training $\mathcal{L}_{\mathcal{D}_{\rm train}}(\boldsymbol{\theta}(t))$ and test loss $\mathcal{L}_{\mathcal{D}_{\rm test}}(\boldsymbol{\theta}(t))$ vs. training epochs $t$.
        We use both GD and full-batch SAM to train the models with NTK parameterization~\citep{jacot2018NTK}.
        We also train a linearized model with GD as baseline.
        % SAM deviate from both GD and linearized model trajectories, demonstrating escape from the lazy regime.
        Results are presented for WideResNets trained on a random subset of 64 images from CIFAR-10 due to the $O(n^2)$ computational complexity of NTK.
    }
    \label{fig:sam_real_world}
    \vspace{-10pt}
  \end{center}
\end{figure*}

%\vspace{1mm}
\noindent
\textbf{Sharpness-Aware Minimization (SAM).}
SAM, introduced by \citet{foret2021sharpnessaware}, has substantially improved the generalization of neural networks.
The core idea of SAM is to perform an inner adversarial gradient perturbation, followed by an outer gradient update.
Its update rule is written as $\boldsymbol{\theta}(t+1) = \boldsymbol{\theta}(t) - \eta \nabla \ell_{\xi_t} \left(\boldsymbol{\theta}(t)+ \rho \frac{\nabla \ell_{\xi_t} (\boldsymbol{\theta}(t))}{\Vert\nabla \ell_{\xi_t} (\boldsymbol{\theta}(t)) \Vert_2}\right),$
where $\xi_t \in [n]$ is the index sampled at iteration $t$ and $\rho$ is the perturbation radius.
Recent works~\citep{Monzio2023sde,zhou2024sharpness} have shown that the effectiveness of SAM stems from its ability to amplify the stochastic noise inherent in SGD.
As label noise SGD also benefits from additional noise, a natural question arises: \begin{center}
    \emph{Can our findings on label noise SGD extend to SAM?} 
\end{center}

\noindent
\textbf{SAM exhibits the two-phase learning dynamics under synthetic setups.}
First, we focus on the exact dynamics of SAM on synthetic data.
We conduct the synthetic experiments similar to previous sections except for replacing label noise SGD with SAM. In \cref{fig:sam_synthetic} \textit{(b)}, the norm of the first-layer weights $\boldsymbol{W}$ initially decreases, demonstrating the progressive diminishing phenomenon; subsequently, the alignment between $\boldsymbol{W}$ and ground-truth interpolator $\boldsymbol{\theta}^{\star}$ increases, indicating an active feature learning process.
Furthermore, in \cref{fig:sam_synthetic}, we visualize the learning dynamics of a single neuron, evidently confirming that the two-phase picture holds for SAM.

%\vspace{1mm}
\noindent
\textbf{SAM promotes feature learning  in practice.}
Second, we explore the properties of SAM in applications.
As in phase I, we compare the loss curves of models trained with different strategies. In \cref{fig:sam_real_world},
the model trained with GD shows loss curves that closely align with its linearized approximations, a hallmark of the lazy regime.
In contrast, training with SAM results in loss curves that deviate significantly from GD trajectories, signaling entry into the rich regime.

In summary, SAM shows similar behaviors to label noise SGD across various setups, suggesting that the underlying ideology of label noise SGD is generalizable.

% \vspace{-7pt}
\section{Conclusion and Outlook}
\label{sec:conclusion}
\vspace{-7pt}

We present an in-depth study on the implicit regularization effect of label noise SGD from empirical observations to theoretical analysis.
Notably, our theory demonstrates the surprising effect of label noise on the oscillation of the second layer, which induces the progressively diminishing phenomenon, leading to the transition from lazy to rich regime.
Our theories are derived from a two-layer linear network. 
% Despite its non-trivialness, it may not fully capture the complexities introduced by non-linearities in modern neural networks.
A future direction is to consider non-linear activation functions and study whether non-linearity influences the training dynamics.
Besides, extending our theory to classification tasks remains an open challenge, which we will explore in future work.

\newpage
% \section*{Impact Statement}
% This paper refers to understanding the role of label noise in the training dynamics of neural networks.
% While our work has the potential to impact society in various ways, we do not identify any specific societal consequences that require particular emphasis at this time.

% \subsubsection*{ACKNOWLEDGMENTS}

{
\small 
\bibliography{ref}
\bibliographystyle{unsrtnat}
}

%%%%%%%%%%%%%%%%%%%%%%%%%%%%%%%%%%%%%%%%%%%%%%%%%%%%%%%%%%%%%%%%%%%%%%%%%%%%%%%
%%%%%%%%%%%%%%%%%%%%%%%%%%%%%%%%%%%%%%%%%%%%%%%%%%%%%%%%%%%%%%%%%%%%%%%%%%%%%%%
% APPENDIX
%%%%%%%%%%%%%%%%%%%%%%%%%%%%%%%%%%%%%%%%%%%%%%%%%%%%%%%%%%%%%%%%%%%%%%%%%%%%%%%
%%%%%%%%%%%%%%%%%%%%%%%%%%%%%%%%%%%%%%%%%%%%%%%%%%%%%%%%%%%%%%%%%%%%%%%%%%%%%%%
\newpage
\appendix
\onecolumn

\section{Preliminaries}

\subsection{Additional notations}
\textbf{Complementary Event.} Let $A$ be an event. We use $\overline{A}$ to denote the complementary event of $A$. We have $\operatorname{Pr}[A]+\operatorname{Pr}[\overline{A}] = 1$.

\begin{definition}
    \textbf{(sub-exponential)} A random variable \( X \) with mean \( \mu = \mathbb{E}[X] \) is sub-exponential if there are non-negative parameters \( (\nu, b) \) such that
\[
\mathbb{E}[e^{\lambda (X - \mu)}] \leq e^{\frac{\nu^2 \lambda^2}{2}} \quad \text{for all} \quad |\lambda| < \frac{1}{b}.
\]
We denote $X\in SE(\nu,b)$.
\end{definition}

% \subsection{Additional assumptions} \label{Additional assumptions}

% For the sake of clarity in our theoretical analysis, we introduce the following two assumptions throughout the paper:

% \textbf{Input magnitude.} The maximum norm of the input samples satisfies $\max_{i}\|\bx_i\|\leq C_{data}$. Here $C_{data}$ is a constant and by definition of $C$ in \cref{cond:main}(A1), we have $C_{data}\leq C \ll m$. \label{Input magnitude}

% \textbf{Dimension of sample.} The dimension of a single sample $d$ satisfies $d\geq \frac{9(\ln2)\cdot K^4}{2c}$, where $K$ and $c$ are defined in \cref{Concentration of the norm}. Notice that $d$ is still a constant and satisfies $d\ll m$. Furthermore, by \cref{Concentration of the norm}, we have
% \begin{equation}
%     \operatorname{Pr}[\left\lvert \| \bx_{i}\| - \sqrt{d}\right\lvert \geq \dfrac{\sqrt{d}}{3}]\leq 2\exp(-\dfrac{4\cdot c\cdot d}{9K^4})\leq\dfrac{1}{2} \label{d is not small}
% \end{equation}

\subsection{Preliminary lemmas}
\begin{lemma}
    Let \( X = \sum_{i=1}^{n} X_i^2 \) where \( X_i \sim N(0,1) \) and i.i.d. Then \( X \in SE(2\sqrt{n}, 4) \). \label{X_i^2}
\end{lemma}

\begin{corollary}
Let \( Z = \sum_{i=1}^{n} X_i\cdot Y_i \) where \( X_i, Y_i \sim N(0,1) \) and  i.i.d. Then \( X \in SE(2\sqrt{n}, 4) \). \label{x_iY_i}
\end{corollary}

\textbf{Proof of \cref{x_iY_i}}. For every $i\in[n]$. we have
\begin{align*}
     \mathbb{E}[e^{\lambda X_i^2 - 1}] &=  \dfrac{1}{\sqrt{2\pi}}\cdot \int_{-\infty}^{+\infty} e^{-\lambda (x^2-1)\cdot e^{-x^2/2}}dx \\
     &= \dfrac{e^{-\lambda}}{\sqrt{1- 2\lambda}}
\end{align*}

And we have 
\begin{align*}
    \mathbb{E}[e^{\lambda X_iY_i}] &= \int_{-\infty}^{+\infty} \int_{-\infty}^{+\infty} e^{\lambda x y} \cdot \frac{1}{2\pi} e^{-\frac{x^2}{2}} e^{-\frac{y^2}{2}} \, dx \, dy \\
     &= \int_{-\infty}^{+\infty} \int_{-\infty}^{+\infty}\frac{1}{2\pi} e^{-\frac{(x-\lambda y)^2}{2}} e^{-\frac{(1-\lambda^2 )y^2}{2}}\, dx \, dy \\
     &= \dfrac{1}{\sqrt{1-\lambda^2}}\int_{-\infty}^{+\infty} \int_{-\infty}^{+\infty} \frac{1}{2\pi} e^{-\frac{x^2}{2}} e^{-\frac{y^2}{2}} \, dx \, dy \\
     &= \dfrac{1}{\sqrt{1-\lambda^2}} \leq  \mathbb{E}[e^{\lambda X^2-1}]
\end{align*}

So we have
\begin{equation*}
    \mathbb{E}[e^{\lambda Z}] = \mathbb{E}[e^{\lambda \sum_{i=1}^n X_i\cdot Y_i}] =\mathbb{E}[\prod_{i=1}^n e^{\lambda X_i\cdot Y_i}] 
\end{equation*}

Since $X_i\cdot Y_i$ are independent, we have 
\begin{align*}
    \mathbb{E}[\prod_{i=1}^n e^{\lambda X_i\cdot Y_i}] &= \prod_{i=1}^n \mathbb{E}[e^{\lambda X_i\cdot Y_i}] \\
    &\leq  \prod_{i=1}^n \mathbb{E}[e^{\lambda X_i^2-1}] = \mathbb{E}[e^{\lambda X}] \\
\end{align*}

By lemma \ref{X_i^2}, we have $X = \sum_{i = 1}^n X_i^2\in SE(2\sqrt{n},4)$. Therefore,  we retain $Z=\sum_{i=1}^n X_i\cdot Y_i \in SE(2\sqrt{n},4)$. \quad\(\square\).
\\[1mm]
\begin{lemma}
    \textbf{(Sub-exponential tail bound)} Suppose that \( X \) is sub-exponential with parameters \( (\nu, b) \). Then
\begin{equation}
    \operatorname{Pr}[X \geq \mu + t] \leq
\exp({-\frac{t^2}{2\nu^2}}) \quad\text{if } 0 \leq t \leq \frac{\nu^2}{b}
\end{equation}
\label{Sub-exponential tail bound}
\end{lemma}

\begin{lemma}
        \textbf{(Concentration of the norm)} Let \( X = (X_1, \dots, X_n) \in \mathbb{R}^n \) be a random vector with independent, sub-gaussian coordinates \( X_i \) that satisfy \( \mathbb{E}X_i^2 = 1 \). Then
\begin{equation}
    \operatorname{Pr} \left\{ \left\lvert\| X \| - \sqrt{n} \geq t \right\lvert\right\} \leq 2 \exp \left( - \frac{ct^2}{K^4} \right) \quad \text{for all } t \geq 0
\end{equation}
where \( K = \max_i \|X_i\|_{\psi_2} \) and \( c \) is an absolute constant. \label{Concentration of the norm}
\end{lemma}

\begin{lemma}
    \textbf{(Chernoff's inequality)} Let \(X_i\) be independent Bernoulli random variables with parameters \(p_i\). Consider their sum \(S_N = \sum_{i=1}^{N} X_i\) and denote its mean by \(\mu = \mathbb{E}[S_N]\). Then, for any \(t > \mu\), we have
\begin{equation}
    \mathbb{P}\{S_N \leq t\} \leq e^{-\mu} \left( \frac{e\mu}{t} \right)^t.
\end{equation} \label{Chernoff's inequality}
\end{lemma}

\section{Phase I: Progressively Diminishing and Escaping the Lazy Regime}

\subsection[Step 1: Bounding a\_i]{Step 1: Bounding $a_i$}
Let $\nabla \hat\ell_{\xi_t}(\btheta(t)) = f(\btheta(t); \bx_{\xi_t}) - y_{\xi_t} - \epsilon_t$.

Using label noise SGD, for any $i\in [m]$, the gradient at time step $t$ is

\begin{gather}
    \dfrac{\partial \hat\ell_{\xi_t}(\btheta(t))}{\partial \bw_{i}^\top(t)} = a_{i}(t)(f(\btheta(t); \bx_{\xi_t}) - y_{\xi_t} - \epsilon_t)\cdot\bx_{\xi_t}^\top \label{update of w_i SGD} \\
     \dfrac{\partial \hat\ell_{\xi_t}(\btheta(t))}{\partial a_{i}(t)} =  (f(\btheta(t); \bx_{\xi_t}) - y_{\xi_t} - \epsilon_t) \cdot\bx_{\xi_t}^\top\cdot\bw_{i}(t) \label{update of a_i SGD}
\end{gather}

Then we retain
\begin{align}
    &\bW(t+1) = \bW(t) - \eta\cdot \ba(t)\cdot(f(\btheta(t); \bx_{\xi_t}) - y_{\xi_t} - \epsilon_t)\cdot\bx_{\xi_t}^\top \label{W(t) change in SGD}   \\[2mm]
    &\ba(t+1) = \ba(t) - \eta\cdot (f(\btheta(t); \bx_{\xi_t}) - y_{\xi_t} - \epsilon_t)\cdot\bW(t)\cdot\bx_{\xi_t} \label{a(t) change in SGD}
\end{align} 

The following \cref{B_0} provides a bound of the initialization of $\|\bw_i\|$ at step 0.
\begin{lemma}
    Let event $B_0 = \{\Vert
     \bm{w_{i}(0)}\Vert\leq 1/4\cdot m^{1/12}\ \text{for all $i\in[m]$}\}$. Suppose \cref{cond:main}~(A1) holds. Under NTK initialization as in \cref{eq:initialization}, we have $\operatorname{Pr}[B_0]\geq 1-O(\frac{1}{m})$. \label{B_0}
\end{lemma}

\textbf{Proof}. Notice that $\bm{w_{i}(0)}\sim \dfrac{1}{\sqrt{d}}N(0,I)$, Let $Z = \sqrt{d}\cdot \bm{w_{i}(0)}$ and $Z\sim N(0,I)$. By  \cref{Concentration of the norm}, we have

\begin{equation}
    \operatorname{Pr}[\left\lvert\Vert
     Z\Vert - \sqrt{d}\right\lvert\geq  \dfrac{1}{8}\cdot m^{1/12}] \leq 2\exp(-\dfrac{c\cdot m^{1/6}}{64K_w^4}) 
\end{equation}
where $K_w$ and $c$ both are positive constant.

Thus 
\begin{equation}
    \operatorname{Pr}[\left\lvert\Vert
     \bm{w_{i}}(0)\Vert - 1\right\lvert\geq \dfrac{m^{1/12}}{8\sqrt{d}}] \leq  2\exp(-\dfrac{c\cdot m^{1/6}}{64K_w^4}) 
\end{equation}

Since $1 \ll m $ and $d\geq 1$, we have 
\begin{equation}
    \operatorname{Pr}[\Vert
     \bm{w_{i}}(0)\Vert \geq \dfrac{1}{4}\cdot m^{1/12}] \leq  2\exp(-\dfrac{c\cdot m^{1/6}}{64K_w^4}) 
\end{equation}

Using union bound, we have 
\begin{equation}
    \operatorname{Pr}[\overline{B_0}] \leq m\cdot2\exp(-\dfrac{c\cdot m^{1/6}}{64K_w^4}) 
\end{equation}

Thus
\begin{equation}
    \operatorname{Pr}[B_0] \geq 1- m\cdot2\exp(-\dfrac{c\cdot m^{1/6}}{64K_w^4})  \geq 1- \dfrac{1}{m}
\end{equation}

The last inequality holds under \cref{cond:main}(A1). \quad\(\square\) \\

\begin{lemma}
    Suppose \cref{cond:main}~(A1-2, 4-6) holds and consider the update rule in \cref{passage:theta update}.  Given the model is still in the lazy regime within $T_0=O(\frac{\log m}{\eta^2\cdot m^{1/2}})$, with probability at least $1-O(\frac{1}{m})$, we have the following two propositions hold:
    
    (i). (\textbf{Bound of Loss}) For every $t\leq T_0$, $ \left\lvert\nabla \hat\ell_{\xi_0}(\btheta(t))\right\lvert  = O(m^{1/4})$.
    
    (ii). (\textbf{Bound of $a_i$}) For every $t\leq T_0$, $a_i(t)
    \leq m^{-1/4}$.
    \label{bound lemma}
\end{lemma}

\noindent
\textbf{Proof.} We prove by induction. 

When step $t=0$, we first prove the (i) holds. When $t = 0$, for any $i\in [m]$, let $\bw_i(0) =\frac{1}{\sqrt{d}}(X_{i1},X_{i2},\cdots,X_{id})$ , $X_{ij}\sim\mathbf{N}(0,1)$ and $\ba(0) =\frac{1}{\sqrt{m}}(Y_1,Y_2,\cdots,Y_m)$, $Y_{i}\sim\mathbf{N}(0,1)$. Let $Z_j = \sum_{i = 1}^m X_{ij}\cdot Y_i$ where $j\in[d]$.

For all $i\in[m],j\in[d]$, $X_{ij},Y_{i}$ are independent,  so we have 
\begin{align}
    &\mathbb{E}[Z_j] = \mathbb{E}[\sum_{i=1}^m X_{ij}Y_{i}] = \sum_{i=1}^m \mathbb{E}[X_{ij}Y_{i}]= \sum_{i=1}^m \mathbb{E}[X_{ij}]\mathbb{E}[Y_{i}] = 0
\end{align}
and 
\begin{align}
    &\operatorname{Var}[Z_j] = \operatorname{Var}[\sum_{i=1}^m X_{ij}Y_{i}] = \sum_{i=1}^m \operatorname{Var}[X_{ij}Y_{i}] =  \sum_{i=1}^m \operatorname{Var}[X_{ij}]\operatorname{Var}[Y_{i}] = m
\end{align}

By Corollary \ref{x_iY_i}, we have for any $j\in[d]$, $Z_j = \sum_{i=1}^m X_{ij}\cdot Y_i\in SE[2\sqrt{m},4]$. By Lemma \ref{Sub-exponential tail bound}, let $t = \sqrt{(8\log m)/m}$ and we retain  
\begin{align*}
    \operatorname{Pr}\left( \left|Z_j = \sum_{i=1}^{m} X_{ij}\cdot Y_i \right| \geq \sqrt{8\log m}\cdot\sqrt{m} \right) 
    &\leq 2 \exp\left(- \dfrac{(\sqrt{8\log m}\cdot\sqrt{m})^2}{2\cdot(2\sqrt{m})^2} \right) \\
    &= \dfrac{2}{m}
\end{align*}

Using union bound, with probability at most $\frac{2d}{m}$, there exists $j\in[d]$ such that  $\left|\sum_{i=1}^{m} X_{ij}\cdot Y_i \right| \geq \sqrt{8\log m}\cdot\sqrt{m}$ 

So with probability $1 - \frac{2d}{m}$, \begin{equation}
    \left|\sum_{i=1}^{m} X_{ij}\cdot Y_i \right| \leq \sqrt{8\log m}\cdot\sqrt{m}\quad \text{for all $j\in[d]$} \label{x_ijY_i}
\end{equation}

Then we have
\begin{align*}
    \Vert\btheta(0)-\btheta^*\Vert &\leq \Vert\btheta(0)\Vert + \Vert\btheta^*\Vert\\[2mm]
&=\Vert\sum_{i=1}^m a_{i}(0)\bw_{i}(0)^\top\Vert + \Vert\btheta^*\Vert\\
&= \dfrac{1}{\sqrt{d}}\cdot\dfrac{1}{\sqrt{m}}\sqrt{\sum_{j=1}^d(\sum_{i=1}^m X_{ij}Y_{i})^2}  + \Vert\btheta^*\Vert\\[2mm]
&\leq \dfrac{1}{\sqrt{d}}\cdot\dfrac{1}{\sqrt{m}}\sqrt{8\cdot d\cdot m\cdot \log m} + m^{-1/4} \quad(\text{by Inequality (\ref{x_ijY_i}) and \cref{cond:main}(A4)}) \\[2mm]
&\leq 3\sqrt{\log m}
\end{align*}
So we have
\begin{align*}
    \left\lvert\nabla \hat\ell_{\xi_0}(\btheta(0))\right\lvert &=\left\lvert\ba(0)\bW(0)\bx_{\xi_0} - y_{\xi_0}-\epsilon_0\right\lvert\\[2mm]
    &\leq \left\lvert\ba(0)\bW(0)\bx_{\xi_0} - y_{\xi_0} \right\lvert +\left\lvert\epsilon_0\right\lvert\\[2mm]
    &\leq \|(\btheta(0)-\btheta^*)\|\cdot\|\bx_{\xi_0}\| +\sigma\\[2mm]
    &\leq3\sqrt{\log m}\cdot C_{data} +\sigma \\[2mm]
\end{align*}
Therefore, with probability at least $1-\frac{2}{m}$, we have $ \left\lvert\nabla \hat\ell_{\xi_0}(\btheta(0))\right\lvert = O(m^{1/4})$, so (i) holds when $t=0$.

Then we prove (ii) holds at step 0. 

Let event $C_t = \{|a_i(t_1)|\leq m^{-1/4} \text{\ for all $t_1\leq t$ and for all $i\in[m]$}\}$. Since $a_i(0) \sim \frac{1}{\sqrt{m}}N(0,1)$, we have
\begin{equation}
    \operatorname{Pr}[|a_i(0)|\geq m^{-1/8}] \leq 2\exp(-\dfrac{c\cdot m^{3/4}}{4K_a^4}) 
\end{equation}
by \cref{Concentration of the norm}, where $K_a$ and $c$ are both positive constant.

Using union bound, we have
\begin{equation}
    \operatorname{Pr}[\overline{C_0}] \leq m\cdot2\exp(-\dfrac{c\cdot m^{3/4}}{4K_a^4})\cdot(1-\dfrac{3d}{m})
\end{equation}

Thus 
\begin{align*}
    \operatorname{Pr}[C_0] = 1- \operatorname{Pr}[\overline{C_0}] &\geq 1 - m\cdot2\exp(-\dfrac{c\cdot m^{3/4}}{4K_a^4}) \cdot (1-\dfrac{3d}{m})  \\
    &\geq (1- \dfrac{1}{m})\cdot(1-\dfrac{3d}{m})  \\
    &\geq (1-\dfrac{4d}{m})
\end{align*}
Since $d$ is a constant, we have $\operatorname{Pr}[C_0] \geq 1-O(\frac{1}{m})$ , which implies (ii) holds when $t=0$.

Assuming the lemma holds at step $t$, we proceed to step $t+1$ where we first establish property (ii) by constructing a super-martingale, and subsequently demonstrate that property (i) also holds.

By definition of online label noise SGD
 algorithm, $x_{\xi_t}, \epsilon_t$ are independent with $\btheta(t)$ , $\bw_i(t)$ and has memorylessness property. Let $\mathcal{F}_0 = \mathbf{\sigma}(\bw_{i}(0), a_i(0) \ \text{for i\ $\in[m]$})$. For $t \geq 1$, let $\mathcal{F}_t =\mathbf{\sigma}(\mathcal{F}_0, \epsilon^{(t)}, \bx^{(t)})$.
Here $\epsilon^{(t)}$ denotes the set $\{\epsilon_0, \epsilon_1,\dots, \epsilon_{t-1}\}$ and $\bx^{(t)}$ denotes the set $\{\bx_{\xi_0}, \bx_{\xi_1}, \dots, \bx_{\xi_{t-1}}\}$.

Obviously, we have $\mathcal{F}_0\subset \mathcal{F}_1\subset \cdots\subset \mathcal{F}_t$. Notice that $\mathbb{E}[\bx_{\xi_t}\cdot \bx_{\xi_t}^\top|\mathcal{F}_t] = I$ and $\mathbb{E}[\epsilon_t|\mathcal{F}_t] = 0$. Additionally, in the lazy regime, for any $i \in [d]$ at step $t$, every neuron holds $\| \bw_i(t) - \bw_i(0)\|\leq \frac{1}{\sqrt{m}}$. So we have

\begin{align*}
   \mathbb{E}[\left\lvert a_{i}(t+1)\right\lvert|\mathcal{F}_t]  &= \left\lvert a_{i}(t) - \eta\cdot \mathbb{E}[(\btheta(t) - \btheta^{\star})\cdot\bm{w_{i}}(t)|\mathcal{F}_t] \right\lvert\\[2mm]
   &= \left\lvert a_{i}(t) - \eta\cdot\mathbb{E}[(\ba(t)^\top\bW(t) - \btheta^{\star})\cdot\bm{w_{i}}(t)|\mathcal{F}_t]\right\lvert \\[2mm]
   &= \left\lvert a_{i}(t) - \eta\cdot\mathbb{E}[(\sum_{j=1}^m \ba_j(t)\bw_j(t)^\top - \btheta^{\star})\cdot \bm{w_{i}}(t)|\mathcal{F}_t]\right\lvert \\ 
   &= \left\lvert a_{i}(t) - \eta\cdot(\sum_{j=1}^m \ba_j(t) \mathbb{E}[\bw_j(0)^\top \cdot \bm{w_{i}}(0)]) + \eta\cdot\btheta^{\star}\bm{w_{i}}(0) + o(\eta/\sqrt{m})\right\lvert\\
   &= \left\lvert (1-\eta)\cdot a_{i}(t) +\eta\cdot\btheta^{\star}\bm{w_{i}}(0) + o(\eta/\sqrt{m})\right\lvert \\
\end{align*}

Let $Y_i(t) = \left\lvert a_{i}(t) - \btheta^*\bw_i(0) - o(\frac{1}{\sqrt{m}})\right\lvert$. Since 

\begin{equation}
    \mathbb{E}[\left\lvert a_{i}(t+1) - \btheta^*\bw_{i}(0)-o(\frac{1}{\sqrt{m}})\right\lvert|\mathcal{F}_t]  = \left\lvert(1-\eta)\cdot (a_{i}(t) -\btheta^{\star}\bw_{i}(0) -o(\frac{1}{\sqrt{m}}))\right\lvert
\end{equation}

We have
\begin{equation}
    \mathbb{E}[Y_i(t+1)|\mathcal{F}_t]  = (1-\eta)\cdot Y_i(t) \leq Y_i(t)
\end{equation}

Therefore, $Y_i(0),Y_i(1),\dots, Y_i(t)$ are super-martingale.

By \cref{B_0}, with probability at least $1-O(\frac{1}{m})$, $\|\bw_i(0)\|\leq m^{1/12}$ for all $i\in [m]$. Conditioned on $\|\bw_i(0)\|\leq m^{1/12}$ and we have 
\begin{align*}
    \left\lvert Y_i(t+1)- Y_i(t)\right\lvert &=  \left\lvert |a_{i}(t+1) - \btheta^*\cdot \bw_{i}(0)- o(\frac{1}{\sqrt{m}})| - |a_{i}(t) - \btheta^*\cdot \bw_{i}(0) - o(\frac{1}{\sqrt{m}})| \right\lvert \\
    &\leq |a_{i}(t+1) - a_{i}(t)| + o(\dfrac{1}{\sqrt{m}})\quad (\text{triangle inequality})\\
    &= |\eta\cdot (\nabla \hat\ell_{\xi_t}(\btheta(t)))\cdot \bx_{\xi_t}^\top \cdot \bw_{i}(t)|+ o(\dfrac{1}{\sqrt{m}}) \\
    &\leq \eta\cdot|\nabla \hat\ell_{\xi_t}(\btheta(t)) |\cdot\Vert\bm{x_{\xi_t}}\Vert\cdot \Vert\bw_{i}(t)\Vert+ o(\dfrac{1}{\sqrt{m}}) \\
    &\leq \eta \cdot |\nabla \hat\ell_{\xi_t}(\btheta(t)) |\cdot\Vert\bm{x_{\xi_t}}\Vert\cdot (\Vert\bw_{i}(0)\Vert + \dfrac{1}{\sqrt{m}})+ o(\dfrac{1}{\sqrt{m}}) \quad{(\text{in the lazy regime})} \\
    &\leq 2\eta \cdot |\nabla \hat\ell_{\xi_t}(\btheta(t)) |\cdot C_{data}\cdot m^{1/12}
\end{align*}

With one-side Azuma's inequality, for any $\lambda>0$, we have
\begin{equation}
    \operatorname{Pr}[Y_i(t+1) - Y_i(0) > \lambda] \leq \exp(-\dfrac{\lambda^2}{(t+1)\cdot 4\eta^2\cdot(C_{data}\cdot|\nabla \hat\ell_{\xi_t}(\btheta(t)) |\cdot m^{1/12})^2})
\end{equation}

Thus 
\begin{equation}
    \operatorname{Pr}[|a_{i}(t+1)| \geq 2\cdot |\btheta^*\cdot \bw_{i}(0)|+ |a_{i}(0)| + \lambda] \leq \exp(-\dfrac{\lambda^2}{(t+1)\cdot 4\eta^2 \cdot|\nabla \hat\ell_{\xi_t}(\btheta(t)) |\cdot(C_{data}\cdot m^{1/12})^2})
\end{equation}

By \cref{cond:main}(A4),  $\|\btheta^*\|\leq \frac{1}{m^{1/3}}$. Then we have $|\btheta^*\cdot \bw_{i}(0)|\leq\frac{1}{4 m^{1/4}}$.  Notice that $\operatorname{Pr}[|a_i(0)|\geq m^{-1/4}] \leq 2\exp(-\frac{c\cdot m^{1/4}}{4K_a^4}) $. Let $\lambda = \frac{1}{2m^{1/4}}$, so with probability at least $(1-2\exp(-\frac{c\cdot m^{1/4}}{4K_a^4}))\cdot(1-\frac{3d}{m})$, we have  
\[2\cdot |\btheta^*\cdot w_{i}(0)|+ |w_{i}(0)| + \lambda] \leq \dfrac{1}{2m^{1/4}} + \dfrac{1}{4m^{1/4}} + \dfrac{1}{4m^{1/4}}\leq \dfrac{1}{m^{1/4}}\]
by \cref{cond:main}(A1).

So we retain
\begin{equation}
    \operatorname{Pr}[|a_{i}(t+1)| \geq m^{-1/4}] \leq \exp(-\dfrac{m^{1/6}}{(t+1)\cdot 8\eta^2\cdot(C_{data}\cdot |\nabla \hat\ell_{\xi_t}(\btheta(t)) |)^2})\cdot(1-\frac{3d}{m})
\end{equation}

By induction hypothesis, we have $|\nabla \hat\ell_{\xi_t}(\btheta(t)) |=O(m^{1/4})$. Using union bound and by \cref{cond:main}(A2), we have
\begin{align*}
    \operatorname{Pr}[\overline{C_t}] &\leq \sum_{j=0}^{t-1} \exp(-\dfrac{m^{1/12}}{j\cdot 8\eta^2\cdot(C_{data}\cdot O(m^{1/4}))^2})\cdot(1-\frac{3d}{m}) \\
    &\leq t\cdot\exp(-\dfrac{m^{1/12}}{t\cdot 8\eta^2\cdot(C_{data}\cdot O(m^{1/4}))^2})\cdot(1-\frac{3d}{m}) \\
    &\leq T_0\cdot\exp(-\dfrac{\sqrt{\eta}^{-1/12}}{T_0\cdot 8\eta^2\cdot(C_{data}\cdot O(m^{1/4}))^2})\cdot(1-\frac{3d}{m}) \\
    &\leq T_0\cdot\exp(-\dfrac{\sqrt{\eta}^{-1/12}}{T_0\cdot 8\eta^2\cdot(C_{data}\cdot O(m^{1/4}))^2})\cdot(1-\frac{3d}{m}) \\
    &= O(\frac{-\ln\eta}{\eta^2})\cdot\exp(-\dfrac{\sqrt{\eta}^{-1/12}}{O(\ln\frac{1}{\eta})\cdot C_{data}^2}\cdot O(1))\cdot(1-\frac{3d}{m}) \\
    &\leq O(\eta)\cdot (1-\dfrac{3d}{m}) = O(\dfrac{1}{m})
\end{align*}
The last inequality is due to \cref{cond:main}(A2). Therefore, we have $
   \operatorname{Pr}[C_t] \geq 1-O(\frac{1}{m})$, which implies (ii) holds.

Then we prove (i) holds. By \cref{def:lazy_regime}, we have $\|\bw_i(t+1)-\bw_i(t)\|\leq m^{-1/2}$. Therefore, with the bound of $a_i(t)$ holds, we retain
\begin{equation}
    \|\btheta(t+1)\| \leq \|\ba(t)^\top\cdot\bW(t)\| \leq \|\ba(0)^\top\cdot\bW(0)\| + m^{-1/2}\cdot m^{-1/4}\cdot m = \|\btheta(0)\| + m^{1/4} 
\end{equation}
Therefore, 
\begin{align*}
    \left\lvert\nabla \hat\ell_{\xi_t}(\btheta(t+1))\right\lvert &=\left\lvert\ba(t+1)\bW(t+1)\bx_{\xi_{t+1}} - y_{\xi_{t+1}}-\epsilon_{t+1}\right\lvert\\[2mm]
    &\leq \left\lvert\ba(t+1)\bW(t+1)\bx_{\xi_{t+1}} - y_{\xi_{t+1}} \right\lvert +\left\lvert\epsilon_{t+1}\right\lvert\\[2mm]
    &\leq \|(\btheta(t+1)-\btheta^*)\|\cdot\|\bx_{\xi_{t+1}}\| +\sigma\\[2mm]
    &\leq 2(\|\btheta(0)\| + m^{1/4} )\cdot C_{data} +\sigma \\[2mm]
    &= O(m^{1/4})
\end{align*}
which implies (i) also holds at step $t+1$. So it follows by induction that the lemma holds. \quad  \( \square\)

\subsection[Step 2: Estimating Delta W\_i]{Step 2: Estimating $\Delta W_i$}

Let $\nabla \hat\ell_{\xi_t}(\btheta(t)) = f(\btheta(t); \bx_{\xi_t}) - y_{\xi_t} - \epsilon_t$, by equ (\ref{update of w_i SGD}) and equ (\ref{update of a_i SGD}), we have 
\begin{gather}
    \label{iter_w1}
    \bw_{i}(t+1) = \bw_{i}(t) - \eta\cdot a_{i}(t)\cdot\nabla \hat\ell_{\xi_t}(\btheta(t))\cdot \bx_{\xi_t}  \\
    \label{iter_w2}
    a_{i}(t+1) = a_{i}(t) - \eta\cdot \nabla \hat\ell_{\xi_t}(\btheta(t))\cdot \bx_{\xi_t}^\top \cdot \bw_{i}(t)
\end{gather}
According to equ (\ref{iter_w1}), after taking the norm on both sides and then square them, we have:

\begin{equation}
        \Vert \bw_{i}(t+1)\Vert^2 = \Vert \bw_{i}(t) \Vert^2 - 2\eta\cdot a_{i}(t)\nabla \hat\ell_{\xi_t}(\btheta(t)) \bx_{\xi_t}^\top\bw_{i}(t) + \Vert\eta\cdot a_{i}(t)\nabla \hat\ell_{\xi_t}(\btheta(t))\bx_{\xi_t}\Vert^2 \label{squere}
\end{equation}

According to equ (\ref{iter_w2}), we have
\begin{equation}
    \label{w2_variable}
        \eta\cdot \nabla \hat\ell_{\xi_t}(\btheta(t))\cdot \bx_{\xi_t}^\top \cdot \bw_{i}(t) = a_{i}(t) -  a_{i}(t+1)
\end{equation}

Substitute equ (\ref{w2_variable}) into the equ (\ref{squere}) and we retain:
\begin{equation}
        \Vert \bw_{i}(t+1)\Vert^2 = \Vert \bw_{i}(t) \Vert^2 - 2\cdot(a_{i}(t) -  a_{i}(t+1))\cdot a_{i}(t) + \Vert\eta\cdot a_{i}(t)\nabla \hat\ell_{\xi_t}(\btheta(t))x_{\xi_t}\Vert^2 
\end{equation}

For any time $T\in \mathbb{N}^+$, summing up from $0$ to $T$ and we have:
\begin{equation}
        \Vert \bw_{i}(T)\Vert^2 = \Vert \bw_{i}(0) \Vert^2 - 2\cdot\sum_{j = 0}^{T-1}(a_{i}(j) -  a_i(j+1))\cdot a_{i}(j) + \sum_{j = 0}^{T-1}\Vert\eta\cdot a_{i}(j)\cdot \nabla \hat\ell_{\xi_j}(\btheta(j))\cdot \bx_{\xi_j}\Vert^2 
\end{equation}

Notice that 
\begin{align}
2\cdot \sum_{j = 0}^{T-1}(a_{i}(j) -  a_{i}(j+1))\cdot a_{i}(j)&= \sum_{j = 0}^{T-1}(a_{i}(j)^2 -  2\cdot a_{i}(j)a_{i}(j+1) + a_{i}(j+1)^2) +a_{i}(0)^2-a_{i}(T)^2\\
&=\sum_{j = 0}^{T-1}(a_{i}(j)-a_{i}(j))^2+a_{i}(0)^2-a_{i}(T)^2
\end{align}

Thus we have
\begin{align*}
        \Vert \bw_{i}(T)\Vert^2 &= \Vert \bw_{i}(0) \Vert^2 - \sum_{j = 0}^{T-1}(a_{i}(j)-a_{i}(j+1))^2-a_{i}(0)^2+a_{i}(T)^2 + \sum_{j = 0}^{T-1}\Vert\eta\cdot \bw_{i}(j)\cdot \nabla \hat\ell_{\xi_j}(\btheta(j))\cdot x_{\xi_j}\Vert^2 \\
        &=\Vert \bw_{i}(0) \Vert^2 - \sum_{j = 0}^{T-1}(\Vert a_{i}(j)-a_{i}(j+1)\Vert^2 - \Vert \bw_{i}(j)-\bw_{i}(j+1)\Vert^2) -a_i(0)^2+a_{i}(T)^2\\
        &=\Vert \bw_{i}(0) \Vert^2 - \sum_{j = 0}^{T-1}\eta^2\cdot\nabla \hat\ell_{\xi_j}(\btheta(j))^2\{(\bx_{\xi_j}^\top \cdot \bw_{i}(j))^2 -  a_{i}(j)^2\cdot\Vert \bx_{\xi_j}\Vert^2\}-a_{i}(0)^2+a_{i}(T)^2
\end{align*}

Let $\Delta W_i(j) = -\nabla \hat\ell_{\xi_j}(\btheta(j))^2\cdot\{(\bx_{\xi_j}^\top \cdot \bw_{i}(j))^2 -  a_{i}(j)^2\cdot\Vert \bx_{\xi_j}\Vert^2\}$. Since $a_i(j)$ is small with high probability, $\Delta W_i(j)$ almost dominates the change of $\Vert \bw_i\Vert^2$ at every step. We have
\begin{equation}
    \Vert \bw_{i}(T)\Vert^2 =\Vert \bw_{i}(0) \Vert^2 + \eta^2\cdot\sum_{j = 0}^{T-1}\Delta W_i(j) -a_{i}(0)^2+a_{i}(T)^2
\end{equation}

\begin{lemma}[Progressively diminishing at each step]\label{norm decay each step}
    Suppose \cref{cond:main}~(A1-2, 4-6) holds and consider the update rule in \cref{passage:theta update}. Given the model is still under the lazy regime at step $T$, then with probability at least $1-O(\frac{1}{m})$, for all the iterative steps $j\leq T_1$ and for every $i\in[m]$: 
    \begin{enumerate}[topsep=0em,leftmargin=0.15in]
        \item $\Delta W_i(j) \leq 0$ with probability at least $1 - \dfrac{\rho}{m^{1/8}}$.
        \item $\Delta W_i(j) \leq -(\dfrac{\sigma}{4})^2$ with probability at least $ \dfrac{1}{4}$.
        \item $\Delta W_i(j) > 0$ with probability at most $\dfrac{\rho}{m^{1/8}}$.
        \item $\Delta W_i(j)\leq O(1)$.
    \end{enumerate}
    where $\rho = \dfrac{2\sqrt{d}}{\sqrt{\pi}}$ is a constant. \label{norm deacy every step}
\end{lemma}

\textbf{Proof.} 

By \cref{bound lemma} (ii), with probability at least $1-O(\frac{1}{m})$, for all $i\in[m]$ and all step $t\leq T_1$, $|a_i(t)|\leq \frac{1}{m^{1/4}}$, i.e. the event $C_j$ happens. All the "Pr" in this lemma conditioned on $C_j$. 

For each $j < T$, we have \[\Delta W_i(j) > 0 \iff (\bx_{\xi_j}^\top\cdot \bw_{i}(j))^2 < a_{i}(j)^2\cdot \Vert \bx_{\xi_j}\Vert^2\iff(\dfrac{\bx_{\xi_j}^\top}{\Vert \bx_{\xi_j}\Vert}\cdot \bw_i(j))^2 < a_{i}(j)^2\]
i.e. 
\begin{equation}
    \left\lvert\boldsymbol{x}^\top\cdot \bw_{i}(j)|<|a_{i}(j)\right\lvert
\end{equation}

where $\boldsymbol{x} = \frac{\bx_{\xi_j}}{\Vert \bx_{\xi_j}\Vert} $ follows a uniform distribution on the n-dimensional unit sphere. 

Let $x_1$ denotes the element in the first dimension of $\boldsymbol{x}$. By symmetry, we have
\begin{equation}
   \operatorname{Pr}[\left\lvert\boldsymbol{x}^\top\cdot \bw_{i}(j)|<|a_{i}(j)\right\lvert] = \operatorname{Pr}[-\dfrac{\Vert a_{i}(j)\Vert}{\Vert \bw_{i}(j)\Vert} < x_1 < \dfrac{\Vert a_{i}(j)\Vert}{\Vert \bw_{i}(j)\Vert}]
\end{equation}
The density of $x_1$ is $f(x_1) = \frac{\Gamma\left(\frac{d}{2}\right)}{\sqrt{\pi} \, \Gamma\left(\frac{d-1}{2}\right)} \left(1 - x_1^2\right)^{\frac{d-3}{2}}$, where $\Gamma$ denotes the gamma function.

Then we have
\begin{align*}
     \operatorname{Pr}[-\dfrac{\Vert a_{i}(j)\Vert}{\Vert \bw_{i}(j)\Vert} < x_1 < \dfrac{\Vert a_{i}(j)\Vert}{\Vert \bw_{i}(j)\Vert}] &= \int_{x_1 =-\frac{\Vert a_{i}(j)\Vert}{\Vert \bw_{i}(j)\Vert}}^\frac{\Vert a_{i}(j)\Vert}{\Vert \bw_{i}(j)\Vert}f(x_1)dx_1\\
     &=\int_{x_1 =-\frac{\Vert a_{i}(j)\Vert}{\Vert \bw_{i}(j)\Vert}}^\frac{\Vert a_{i}(j)\Vert}{\Vert \bw_{i}(j)\Vert}\frac{\Gamma\left(\frac{d}{2}\right)}{\sqrt{\pi} \, \Gamma\left(\frac{d-1}{2}\right)} \left(1 - x_1^2\right)^{\frac{d-3}{2}}dx_1 \\
     &\leq  \int_{x_1 =-\frac{\Vert a_{i}(j)\Vert}{\Vert \bw_{i}(j)\Vert}}^\frac{\Vert a_{i}(j)\Vert}{\Vert \bw_{i}(j)\Vert}\frac{\Gamma\left(\frac{d}{2}\right)}{\sqrt{\pi} \, \Gamma\left(\frac{d-1}{2}\right)} dx_1 \\
     &= 2\cdot\dfrac{\Vert a_{i}(j)\Vert}{\Vert \bw_{i}(j)\Vert}\cdot \frac{\Gamma\left(\frac{d}{2}\right)}{\sqrt{\pi} \, \Gamma\left(\frac{d-1}{2}\right)}\\ 
     &\leq \frac{2\sqrt{d}}{\sqrt{\pi}}\cdot\dfrac{\Vert a_{j}(j)\Vert}{\Vert \bw_{i}(j)\Vert} \\
     &\leq \dfrac{\rho}{m^{1/8}}
\end{align*}

The second-to-last inequality is because $\frac{\Gamma\left(\frac{d}{2}\right)}{\Gamma\left(\frac{d-1}{2}\right)}\leq \sqrt{d}$. So we retain $\Delta W_i(j) \leq 0$ with probability at least $1 - \frac{\rho}{m^{1/8}}$
and $\Delta W_i(j) > 0$ with probability at most $\frac{\rho}{m^{1/8}}$.

By Lemma \ref{bound lemma} (i) , $\nabla \hat\ell_{\xi_j}(\btheta(j))\leq O(m^{1/4})$. So with \cref{cond:main}A(5), $\Vert \bx_{\xi_j}\Vert\leq C_{data}$  we have
\begin{align*}
    \Delta W_i(j) &= -\nabla \hat\ell_{\xi_j}(\btheta(j))^2\cdot(\bx_{\xi_j}^\top \cdot \bw_{i}(j))^2 + \nabla \hat\ell_{\xi_j}(\btheta(j))^2\cdot(a_{i}(j)^2\cdot\Vert \bx_{\xi_j}\Vert^2)\\[2mm]
    &\leq \nabla \hat\ell_{\xi_j}(\btheta(j))^2\cdot(a_{i}(j)^2\cdot\Vert \bx_{\xi_j}\Vert^2) \\[2mm]
    &\leq \dfrac{O(m^{1/4})^2\cdot C_{data}^2}{m^{1/2}}  = O(1)
\end{align*}
Finally, we prove $\operatorname{Pr}[\Delta W_i(j) \leq -(\dfrac{\sigma}{4})^2]\geq \dfrac{1}{4}$. We have 
\begin{align*}
&\scalebox{0.8}{}{\operatorname{Pr}[\Delta W_i(j) \leq -(\dfrac{\sigma}{4})^2]=\operatorname{Pr}[\nabla \hat\ell_{\xi_j}(\btheta(j))^2\{(\bx_{\xi_j}^\top \cdot \bw_{i}(j))^2 -  a_{i}(j)^2\cdot\Vert \bx_{\xi_j}\Vert^2\}\geq (\dfrac{\sigma}{4})^2]} \\
    &=\scalebox{0.5}{}{\operatorname{Pr}[\nabla \hat\ell_{\xi_j}(\btheta(j))^2\{(\bx_{\xi_j}^\top \cdot \bw_{i}(j))^2 -  a_{i}(j)^2\cdot\Vert \bx_{\xi_j}\Vert^2\}\geq (\dfrac{\sigma}{4})^2|\ \bx_{\xi_j} : |\dfrac{\bx_{\xi_j}^\top}{\Vert \bx_{\xi_j}\Vert}\cdot \bw_{i}(j)|<|a_{i}(j)|] \cdot \operatorname{Pr}[|\bx_{\xi_j}^\top\cdot \bw_{i}(j)|<|a_{i}(j)|]}\\
    &+\scalebox{0.5}{}{\operatorname{Pr}[\nabla \hat\ell_{\xi_j}(\btheta(j))^2\{(\bx_{\xi_j}^\top \cdot \bw_{i}(j))^2 -  a_{i}(j)^2\cdot\Vert \bx_{\xi_j}\Vert^2\geq (\dfrac{\sigma}{4})^2|\ \bx_{\xi_j}: |\dfrac{\bx_{\xi_j}^\top}{\Vert \bx_{\xi_j}\Vert}\cdot \bw_{i}(j)|\geq|a_{i}(j)|] \cdot \operatorname{Pr}[|\bx_{\xi_j}^\top\cdot \bw_{i}(j)|\geq|a_{i}(j)|]}   \\
    &\geq \scalebox{0.5}{}{\operatorname{Pr}[\nabla \hat\ell_{\xi_j}(\btheta(j))^2\{(\bx_{\xi_j}^\top \cdot \bw_{i}(j))^2 -  a_{i}(j)^2\cdot\Vert \bx_{\xi_j}\Vert^2\}\geq(\dfrac{\sigma}{4})^2|\ \bx_{\xi_j}:  |\dfrac{\bx_{\xi_j}^\top}{\Vert \bx_{\xi_j}\Vert}\cdot \bw_{i}(j)|\geq|a_{i}(j)|] \cdot \operatorname{Pr}[|\bx_{\xi_j}^\top\cdot \bw_{i}(j)|\geq|a_{i}(j)|]} \\
     &\geq \scalebox{0.7}{}{(1-\dfrac{\rho}{m^{1/8}})\cdot\operatorname{Pr}[\nabla \hat\ell_{\xi_j}(\btheta(j))^2\{(\bx_{\xi_j}^\top \cdot \bw_{i}(j))^2 -  a_{i}(j)^2\cdot\Vert \bx_{\xi_j}\Vert^2\}\geq\ (\dfrac{\sigma}{4})^2|\ \bx_{\xi_j}\ :|\dfrac{\bx_{\xi_j}^\top}{\Vert \bx_{\xi_j}\Vert}\cdot \bw_{i}(j)|\geq|a_{i}(j)|]}  
\end{align*}
Since $\nabla \hat\ell_{\xi_j}(\btheta(j))^2 = ((\btheta(j) - \btheta^*)\bx_{\xi_j} - \epsilon_j)^2$ and $\epsilon_j$ is chosen uniformly,the probability that $(\btheta(j) - \btheta^*)\bx_{\xi_j}$ and $\epsilon_j$ have the same sign is at least  $\frac{1}{2}$. If the two elements have the same sign, then $\nabla \hat\ell_{\xi_j}(\btheta(j))^2 \geq \sigma^2 $. So we have 
\begin{align*}
        &\operatorname{Pr}[\nabla \hat\ell_{\xi_j}(\btheta(j))^2\{(\bx_{\xi_j}^\top \cdot \bw_{i}(j))^2 -  a_{i}(j)^2\cdot\Vert \bx_{\xi_j}\Vert^2\}\geq\ (\dfrac{\sigma}{4})^2|\ \bx_{\xi_j}\ \ s.t.\ \ |\dfrac{\bx_{\xi_j}^\top}{\Vert \bx_{\xi_j}\Vert}\cdot \bw_{i}(j)|\geq|a_{i}(j)|]  \\
       &\geq\dfrac{1}{2}\cdot\operatorname{Pr}[\{(\bx_{\xi_j}^\top \cdot \bw_{i}(j))^2 -  a_{i}(j)^2\cdot\Vert \bx_{\xi_j}\Vert^2\}\geq (\dfrac{1}{4})^2|\ \bx_{\xi_j}\ \ s.t.\ \ |\dfrac{\bx_{\xi_j}^\top}{\Vert \bx_{\xi_j}\Vert}\cdot \bw_{i}(j)|\geq|a_{i}(j)|]  \\
        &\geq\dfrac{1}{2}\cdot\operatorname{Pr}[\{(\bx_{\xi_j}^\top \cdot \bw_{i}(j))^2 -  \dfrac{1}{m^{1/4}}\cdot\Vert \bx_{\xi_j}\Vert^2\}\geq (\dfrac{1}{4})^2|\ \bx_{\xi_j}\ \ s.t.\ \ |\dfrac{\bx_{\xi_j}^\top}{\Vert \bx_{\xi_j}\Vert}\cdot \bw_{i}(j)|\geq|a_{i}(j)|]    \quad\text{($a_i(j)^2\leq m^{1/4}$)}\\
          &=\dfrac{1}{2}\cdot\operatorname{Pr}[\left\lvert\dfrac{\bx_{\xi_j}^\top }{\Vert \bx_{\xi_j}\Vert}\cdot \bw_{i}(j)\right\lvert\geq \sqrt{(\dfrac{1}{4\cdot\Vert \bx_{\xi_j}\Vert})^2+\dfrac{1}{m^{1/4}}}|\ \bx_{\xi_j}\ \ s.t.\ \ |\dfrac{\bx_{\xi_j}^\top}{\Vert \bx_{\xi_j}\Vert}\cdot \bw_{i}(j)|\geq|a_{i}(j)|]  \\
          &=\dfrac{1}{2}\cdot\operatorname{Pr}[\left\lvert\dfrac{\bx_{\xi_j}^\top }{\Vert \bx_{\xi_j}\Vert}\cdot \bw_{i}(j)\right\lvert\geq \sqrt{(\dfrac{1}{4\cdot\Vert \bx_{\xi_j}\Vert})^2+\dfrac{1}{m^{1/4}}}\ ]  \quad\quad(|a_i(j)|\leq\dfrac{1}{m^{1/8}}\leq \sqrt{\dfrac{1}{\Vert \bx_{\xi_j}\Vert^2}+\dfrac{1}{m^{1/4}}}\ )\\
         &\geq \dfrac{1}{2}\cdot\operatorname{Pr}[\left\lvert\dfrac{\bx_{\xi_j}^\top }{\Vert \bx_{\xi_j}\Vert}\cdot \bw_{i}(j)\right\lvert\geq \sqrt{(\dfrac{1}{4\cdot\Vert \bx_{\xi_j}\Vert})^2\cdot\dfrac{3}{2}}\ ] 
\end{align*}
By \cref{cond:main}A(6), we have
\begin{equation}
    \operatorname{Pr}[\left\lvert \| \bx_{i}\| - \sqrt{d}\right\lvert \geq \dfrac{\sqrt{d}}{3}]\leq 2\exp(-\dfrac{4\cdot c\cdot d}{9K^4})\leq\dfrac{1}{2} \label{d is not small}
\end{equation}
which implies $\operatorname{Pr}[\|\bx_{\xi_j}\|\geq \frac{2}{3\sqrt{d}}] \geq \frac{2}{3}$. So we have
\begin{align*}
    \operatorname{Pr}[\left\lvert\dfrac{\bx_{\xi_j}^\top }{\Vert \bx_{\xi_j}\Vert}\cdot \bw_{i}(j)\right\lvert\geq \sqrt{(\dfrac{1}{4\cdot\Vert \bx_{\xi_j}\Vert})^2\cdot\dfrac{3}{2}}\ ] &\geq \operatorname{Pr}[\left\lvert\dfrac{\bx_{\xi_j}^\top }{\Vert \bx_{\xi_j}\Vert}\cdot \bw_{i}(j)\right\lvert\geq \dfrac{3}{2\sqrt{d}}\cdot \dfrac{1}{4\cdot\Vert \bx_{\xi_j}\Vert}] \\
    &= 1 - \operatorname{Pr}[\left\lvert\dfrac{\bx_{\xi_j}^\top }{\Vert \bx_{\xi_j}\Vert}\cdot \bw_{i}(j)\right\lvert\leq \dfrac{3}{8\sqrt{d}}] \\
     &\geq 1 - \rho\cdot\dfrac{3}{8\sqrt{d}} \\
     &=  1 - \dfrac{2\sqrt{d}}{\sqrt{\pi}}\cdot\dfrac{3}{8\sqrt{d}}= 1-\dfrac{3}{4\sqrt{\pi}} 
\end{align*}

By \cref{cond:main}(A1), $m\geq ((1-\dfrac{3}{4\sqrt{\pi}})\cdot\rho/(\dfrac{1}{2}-\dfrac{3}{4\sqrt{\pi}}))^8$. Therefore, we have
\begin{align*}
    \operatorname{Pr}[\Delta W_i(j) \leq -(\dfrac{\sigma}{4})^2] &\geq (1-\dfrac{\rho}{m^{1/8}})\cdot \dfrac{1}{2}\cdot(1-\dfrac{3}{4\sqrt{\pi}})  \\
    &= \dfrac{1}{4} + \dfrac{1}{2}\cdot((\dfrac{1}{2}-\dfrac{3}{4\sqrt{\pi}})-(1-\dfrac{3}{4\sqrt{\pi}})\cdot\dfrac{\rho}{m^{1/8}})\\
    &\geq \dfrac{1}{4}
\end{align*}
which completes the proof. \quad\(\square\)

\begin{theorem}[Escaping the lazy regime]
    Suppose \cref{cond:main}~(A1-2, 4-6) holds and consider the update rule in \cref{passage:theta update}.
    With probability at least $1-O(\frac{1}{m})$, all the neurons $\bw_{i}$ ($i\in[m]$) escape from the lazy regime at time $T_1 = \frac{384\sqrt{\log m}}{\sigma^2\eta^2\sqrt{m}}$. \label{Phase I main theorem}
\end{theorem} 

\noindent
\textbf{Proof.} Let event $C_t = \{|a_i(t_1)|\leq m^{-1/4} \text{\ for all $t_1\leq t$ and for all $i\in[m]$}\}$.  By \cref{bound lemma} (ii), with probability at least $1-O(\frac{1}{m})$, the event $C_{T_1}$ happens. We assume $C_{T_1}$ happens and all the "Pr" in this theorem conditioned on $C_j$. 

In the following, we will provide a proof by contradiction. Assume with probability at least $O(\frac{1}{m})$, there exists some neurons $\bw_i$ ($i\in[m]$) s.t. $\|\bw_i(t) - \bw_i(0)\|\leq \frac{1}{\sqrt{m}}$ holds for all $T_1$ steps. By  \cref{Concentration of the norm}, we have
\begin{equation}
    \operatorname{Pr} \left\{ \left\lvert\| \bw_{i}(0) \| - \sqrt{2\log m} \geq t \right\lvert\right\} \leq 2 \exp \left( - \frac{c\cdot\log m}{k^4} \right) = O(\dfrac{1}{m})
\end{equation}
where $k$ and \( c \) is an absolute constant. 

Using union bound, with probability at least $1-O(\frac{1}{m})$, for these neurons stuck in the lazy regime,  $\Vert \bw_i(0)\Vert\leq \sqrt{2\log m}$.

We view $\Delta W_i(j)$ as a random variable.  We use $\Omega$ to denote the whole sample space. Let event \[\Omega_i = \{\omega\in\Omega\ |\Delta W_i(j)(\omega)\leq -(\dfrac {\sigma}{4})^2 \} \]

For any $\omega\in\Omega$, by  \cref{norm deacy every step}, we have $\operatorname{Pr}[\Omega_i]\geq\frac{1}{4}$. Since the event $C_{T_1}$ happens and $\bw_{i}$ stays in the lazy regime, the estimation of the lower bound of $\Delta W_i(j)$ only depends on the randomness of $\bx_{\xi_j}$. Thus, there exists a subset of $\Omega_i$  we denote $\Theta_i$ such that $$\Theta_i  = \{\omega\in\Omega_i |\Delta W_i(\omega)\leq  -(\sigma/4)^2, \text{$\|\bw_{i}(j) - \bw_{i}(0)\|\leq\frac{1}{\sqrt{m}},|a_{i}(j)|\leq 1/m^{1/8}$}\}$$
and $\operatorname{Pr}[\Theta_i]=\frac{1}{4}$. For $\Theta_i$ ($i\in [T_1-1]$), the random  randomness only depends on $\bx_{\xi_t}$. Therefore, $\Theta_0,\Theta_1,\cdots,\Theta_{T_1-1}$ are mutual independent events.

Then we define indicator random variable $X_0, X_1,\cdots,X_{T_1-1}$ as follow: for every $\omega\in\Omega$, for any $j\in [T_1-1]$, 
\[X_j(\omega) = \left\{
		\begin{array}{l}
                   1 \quad\text{if $\omega\in\Theta_i$.}\\
                   0 \quad\text{otherwise.}\\     
		\end{array}
		\right.\]
So we have $X_0,X_1,\cdots,X_{T_1-1}$ are independent and $\operatorname{Pr}[X_j(\omega) = 1] = \frac{1}{4}$. And we have
\begin{equation}
    \mathbb{E}[\sum_{j=0}^{T_1-1}X_j] = \sum_{j=0}^{T_1-1}\mathbb{E}[X_j] = \dfrac{T_1}{4}
\end{equation}
Then we have
\begin{align*}
    &\operatorname{Pr}[\eta^2\cdot\sum_{j=0}^{T_1-1}\Delta W_i(j)\leq \dfrac{-2\sqrt{2\log m}}{\sqrt{m}}] \\
    &= \operatorname{Pr}[\eta^2\cdot(\sum_{j=0}^{T_1-1}\Delta W_i(j)\cdot \mathbb{I}[\Delta W_i(j) > 0] + \sum_{j=0}^{T_1-1}\Delta W_i(j)\cdot \mathbb{I}[\Delta W_i(j) \leq 0])\leq \dfrac{-2\sqrt{2\log m}}{\sqrt{m}}] \\
    &\geq \operatorname{Pr}[\eta^2\cdot(\sum_{j=0}^{T_1-1}\Delta W_i(j)\cdot O(\dfrac{1}{m^{1/8}}) + \sum_{j=0}^{T_1-1}\Delta W_i(j)\cdot \mathbb{I}[\Delta W_i(j) \leq -(\dfrac{\sigma}{4})^2])\leq \dfrac{-2\sqrt{2\log m}}{\sqrt{m}}] \\
    &\geq \operatorname{Pr}[\eta^2 \cdot T_1\cdot O(1)\cdot O(\dfrac{1}{m^{1/8}}) - \eta^2\cdot(\dfrac{\sigma}{4})^2\sum_{j=0}^{T-1}X_j\leq \dfrac{-2\sqrt{2\log m}}{\sqrt{m}}] \\
    &= \operatorname{Pr}[- \eta^2\cdot(\dfrac{\sigma}{4})^2\cdot\sum_{j=0}^{T_1-1}X_j\leq \dfrac{-2\sqrt{2\log m}}{\sqrt{m}} - o(\dfrac{1}{\sqrt{m}})] \quad(\text{$T_1\cdot\eta^2 = O(\dfrac{1}{\sqrt{m}}$}))\\
    &\geq \operatorname{Pr}[- \eta^2\cdot(\dfrac{\sigma}{4})^2\cdot\sum_{j=0}^{T_1-1}X_j\leq -\dfrac{3\sqrt{\log m}}{\sqrt{m}}]\\
    &= \operatorname{Pr}[\sum_{j=0}^{T_1-1}X_i \geq\dfrac{48\sqrt{\log m}}{\sigma^2\eta^2\sqrt{m}}] = \operatorname{Pr}[\sum_{j=0}^{T_1-1}X_i \geq\dfrac{T_1}{8}]\\
    &\geq 1 - \exp(-\dfrac{T_1}{4})\cdot(\dfrac{e\cdot T_1/4}{T_1/8})^{T_1/8} = 1-(\dfrac{2}{e})^{T_1/8}
\end{align*}
The last inequality is by \cref{Chernoff's inequality}. 

Since $T_1 = \frac{1}{\eta^2\cdot \sqrt{m}} \gg m$ and $a_i(T)^2\leq\frac{1}{\sqrt{m}}$, with probability at least $(1-O(\frac{1}{m}))\cdot(1-(\frac{2}{e})^{T_1/8}) = 1-O(\frac{1}{m})$, we have
\begin{align*}
     \Vert \bw_{i}(T)\Vert^2 - \Vert \bw_{i}(0) \Vert^2 &= \eta^2\cdot\sum_{j = 0}^{T-1}\Delta W_i(j) -a_{i}(0)^2 +a_{i}(T)^2 \\
    &\leq \eta^2\cdot\sum_{j = 0}^{T-1}\Delta W_i(j) + \dfrac{1}{\sqrt{m}} \\
    &\leq -2\dfrac{\sqrt{2\log m}}{\sqrt{m}} + \dfrac{1}{\sqrt{m}} \leq -\dfrac{2\sqrt{\log m}}{\sqrt{m}}
\end{align*}
Thus
\begin{align*}
    \Vert \bw_{i}(T)\Vert - \Vert \bw_{i}(0) \Vert &\leq -\dfrac{2\sqrt{\log m}}{\sqrt{m}}\cdot\dfrac{1}{\Vert \bw_{i}(T)\Vert + \Vert \bw_{i}(0) \Vert} \\
    &\leq -\dfrac{2\sqrt{\log m}}{\sqrt{m}}\cdot \dfrac{1}{\frac{1}{\sqrt{m}} + \sqrt{2\log m}} \\
     &< -\dfrac{2\sqrt{\log m}}{\sqrt{m}}\cdot \dfrac{1}{2\sqrt{\log m}} = -\dfrac{1}{\sqrt{m}}
\end{align*}
Which is a contradiction to the definition of lazy regime! Therefore, we complete the proof. \quad \(\square\)

\subsection{Algorithm 2: oscillation of the second layer within low range}
We design a new update rule as follow: for every neuron $\bw_i$ and $a_i$ at step $t$, we have
\begin{gather}
    \label{iter_w1_eta}
    \bw_{i}(t+1) = \bw_{i}(t) - \eta\cdot a_{i}(t)\cdot\dfrac{1}{n}\sum_{i=1}^n(\ba(t)^\top\bW(t)\bx_{i} - y_{i})\cdot \bx_{i} \\
    \label{iter_w2_eta}
    a_{i}(t+1) = a_{i}(t)  + \delta_i(t)
\end{gather}
where 
\begin{equation}
    \delta_i(t)=\left\{
\begin{aligned}
& -\eta^{0.25} \quad \text{if $a_{i}(t)= \eta^{0.25}$}\\
& \eta^{0.25} \quad\quad \text{if $a_{i}(t)= -\eta^{0.25}$} \\
& \text{$\sim\{-\eta^{0.25}, \eta^{0.25}\}$}
\ \ \ \ \text{if $a_{i}(t)= 0$}
\end{aligned}
\right.
\end{equation}
Besides, we design a new initialization for $\ba$: for every $i\in [m]$, we have 
\begin{equation}
    a_i(0)=\left\{
\begin{aligned}
& -\eta^{0.25} \quad \text{with probability 1/4}\\
& \eta^{0.25} \ \quad\quad \text{with probability 1/4} \\
& 0\quad\quad\quad\ \ \  \text{with probability 1/2}
\end{aligned}
\right.
\end{equation}
Notice that $a_i(0)$ follows the stationary distribution. Since 
the transition matrix of $a_i$ is \[
Q = \begin{bmatrix}
0 & 1 & 0  \\
\frac{1}{2} & 0 & \frac{1}{2} \\
0 & 1 & 0 \\
\end{bmatrix}
\]
it is easy to verify that $[1/4,1/2,1/4] = [1/4,1/2,1/4]\cdot Q$.

Thus by large number law, we can approximate $\Vert\ba(t)\Vert^2$ by
\begin{equation}
        \Vert\ba(t)\Vert^2 = \sum_{i=1}^m |a_i(t)|^2= m\cdot\mathbb{E}[|a_i(t)|^2]= \dfrac{m\cdot\eta^{0.5}}{2}
\end{equation}
\begin{lemma}
    Using algorithm 2, we have \begin{equation}
    \mathbb{E}[\btheta(t)] =\mathbb{E}[\btheta(0)]
\end{equation}\label{expectation of theta(t)}
\end{lemma}
\textbf{Proof.} Notice that
\begin{gather} 
    \bW(t+1)= \bW(t+1)- \eta\cdot \ba(t)\cdot\dfrac{1}{n}\sum_{i=1}^n(\ba(t)^\top\bW(t)\bm{x_i} - y_i)\cdot \bm{x_i^\top}   \\
    \ba(t+1) = \ba(t) + \bm{\delta}(t)^\top
\end{gather}

Where $\bm{\delta}(t) = [\delta_1(t),\delta_2(t),\cdots,\delta_m(t)]^\top$.

Notice that $\ba(t+1)^\top\cdot\ba(t) = 0$, then we have
\begin{align}
    \btheta(t+1) &= \ba(t+1)^\top\cdot \bW(t+1) \\
    &= (\ba(t)^\top +\bm{\delta}(t))\cdot(\bW(t) - \eta\cdot \ba(t)\cdot \dfrac{1}{n}\sum_{i=1}^n(\ba(t)^\top\bW(t)\bx_i - y_i)\cdot \bx_i^\top)\\
    &= \btheta(t) + \bm{\delta}(t)\cdot \bW(t) -(\ba(t+1)^\top \cdot\ba(t))\cdot \eta\cdot\dfrac{1}{n}\sum_{i=1}^n(\ba(t)^\top\bW(t)\bx_i - y_i)\cdot \bx_i^\top\\
    &= \btheta(t) + \bm{\delta}(t)\cdot \bW(t) \label{theta_equation}
\end{align}
Since $\mathbb{E}[\bm{\delta}(t)] = \bm{0}$, we have
\begin{align*}
        \mathbb{E}[\btheta(t)] &= \mathbb{E}[\btheta(0)] 
\end{align*}

\subsection{Progressive diminishing of norm using Algorithm 2}

Since $\bx_i\sim N(0,I)$ i.i.d, by law of large numbers, $\frac{1}{n}\sum_{i=1}^n\bx_i\cdot \bx_i^\top$ is approximately to $\mathbb{E}[\bx_i\cdot\bx_i^\top] = I$. So we have 
\begin{gather}
    \bm{w_{i}}(t+1) = \bm{w_{i}}(t) - \eta\cdot a_{i}(t)\cdot(\btheta(t) - \btheta^*)^\top \\
    a_{i}(t+1) = a_{i}(t)  + \delta_i(t)
\end{gather}

Square both sides of the equation and we have
\begin{align}
    \Vert \bw_{i}(t+1)\Vert^2 = \Vert \bw_{i}(t)\Vert^2 - 2\eta\cdot a_{i}(t)\cdot(\btheta(t)-\btheta^*)^\top\cdot\bw_{i}(t+1) + \eta^2\cdot a_{i}(t)^2\cdot\Vert\btheta(t)-\btheta^*\Vert^2 
\end{align}

Since $|a_{i}(t)|\leq \eta^{0.25}$, we have 
\begin{align}
    \dfrac{1}{m}\sum_{i=1}^m \Vert \bw_{i}(t+1)\Vert^2 &= \dfrac{1}{m}\sum_{i=1}^m \Vert \bw_{i}(t)\Vert^2 - 2\eta\cdot(\btheta(t)-\btheta^*)^\top\cdot\dfrac{1}{m}\sum_{i=1}^m a_{i}(t)\bw_{i}(t) + \eta^2\cdot\dfrac{1}{m}\sum_{i=1}^m a_{i}(t)^2\Vert \btheta(t)-\btheta^*\Vert^2 \\
    &= \dfrac{1}{m}\sum_{i=1}^m \Vert \bw_{i}(t)\Vert^2 - 2\eta\cdot\dfrac{1}{m}\cdot(\btheta(t)-\btheta^*)^\top\cdot \btheta(t) + \eta^2\cdot\dfrac{1}{m}\sum_{i=1}^m a_{i}(t)^2\Vert \btheta(t)-\btheta^*\Vert^2 \\
    &\leq \dfrac{1}{m}\sum_{i=1}^m \Vert \bw_{i}(t)\Vert^2 - 2\eta\cdot\dfrac{1}{m}\cdot(\btheta(t)-\btheta^*)^\top\cdot\btheta(t) + \eta^2\cdot(\eta^{0.25})^2\Vert \btheta(t)-\btheta^*\Vert^2\\
    &= \dfrac{1}{m}\sum_{i=1}^m \Vert \bw_{i}(t)\Vert^2 - (\dfrac{2\eta}{m} - \eta^{2.5})\Vert(\btheta(t)-\btheta^*)\Vert^2 - \dfrac{2\eta}{m}\cdot(\btheta(t)-\btheta^*)^\top\cdot \btheta^* \label{relation}\\
\end{align}

\begin{lemma}
    If we set $a_{i}(0)$ to stationary distribution, then for any $t,k\in\mathbf{N}^+$, we have
\begin{equation}
\mathbf{E}[\delta_i(t)\cdot\delta_i(t+k)] =\left\{
\begin{aligned}
& -\dfrac{\eta^{0.5}}{2}\quad \text{if $k=1$}\\
& 0  \quad\quad\quad\quad \text{if $k>1$}
\end{aligned}
\right.
\end{equation}\label{Markov Stationary}
\end{lemma}

\textbf{Proof.} Without loss of generality, assume we know that $\delta_i(t) = \eta^{0.25}$, then the distribution changes to $(0,\frac{1}{2},\frac{1}{2})$. So we have $$\mathbb{E}[\delta_i(t)\cdot\delta_i(t+1)] = \dfrac{1}{2}\cdot \eta^{0.25}\cdot (-\eta^{0.25}) + \dfrac{1}{2}\cdot 0 = -\dfrac{\eta^{0.5}}{2}$$

Then just after that, the distribution changes from $(0,\frac{1}{2},\frac{1}{2})$ to stationary distribution $(\frac{1}{4},\frac{1}{2},\frac{1}{4})$ again. Thus for any $k>1$ we have
$$\mathbb{E}[\delta_i(t)\cdot \delta_i(t+k)] = 0 $$

\begin{lemma}[Progressively diminishing under simulation setup]\label{algorithm 2 decay}
     Suppose \cref{cond:main}~(A1-3, 5-6) holds (let $m = \frac{1}{\sqrt{\eta}}$) and consider the update rule in \cref{iter_w1_eta,iter_w2_eta}, there exists a step $t_0\leq \frac{1}{\eta^2}$ such that
     \begin{equation}
        \mathbb{E}[\dfrac{1}{m}\sum_{i=1}^m \Vert \bw_{i}(t_0)\Vert^2] \leq\sqrt{\eta}
    \end{equation}
\end{lemma}

\textbf{Proof. }
 Firstly, we give a lower bound of $\mathbb{E}[\Vert \btheta(t)-\btheta^*\Vert^2]$, which shows the decrease of each step.
According to equ (\ref{theta_equation}), we retain
\begin{equation}
    \btheta(t)-\btheta^* = \btheta(0) - \btheta^* + \sum_{i=0}^{t-1}\bm{\delta}(i)\cdot \bW(i) 
\end{equation}

And Since $\mathbb{E}[\bm{\delta}(i)] = 0$, we have
\begin{equation}
    \mathbb{E}[(\btheta(0)-\btheta^*)\cdot\sum_{i=0}^{t-1}\bm{\delta}(i)\cdot \bW(i)] = \mathbb{E}[\bm{\delta}(i)]\cdot\mathbb{E}[(\btheta(0)-\btheta^*)\cdot\sum_{i=0}^{t-1}\bW(i)] = 0
\end{equation} 

So we have 
\begin{equation}    
    \mathbb{E}[\Vert(\btheta(t)-\btheta^*)\Vert^2] =
   \mathbb{E}[\Vert(\btheta(0)-\btheta^*)\Vert^2] + (\eta^{0.25})^2\cdot\sum_{i=0}^{t-1}\mathbb{E}[\Vert \bW(i)\Vert^2]  + 2\sum_{i=0}^{t-1}\sum_{j<i}\mathbb{E}[(\bm{\delta}(i)\bW(i))\cdot(\bm{\delta}(j) \bW(j))^\top] \label{theta^2}\\
\end{equation}

Notice that
\begin{align*}
    \mathbb{E}[(\bm{\delta}(i)\bW(i))\cdot(\bm{\delta}(i+1) \bW(i+1)^\top] &=\mathbb{E}[\bm{\delta}(i)(\bW(i)\bW(i+1)^\top)\bm{\delta}(i+1)^\top ] \\[2mm]
    &=\sum_{1\leq l,r\leq m}\mathbb{E}[\delta_{l}(i)\delta_{r}(i+1)\cdot (\bW(i)\bW(i+1)^\top)[l,r]] \\[2mm]
    &= \sum_{1\leq l,r\leq m} \mathbb{E}[\delta_{l}(i)\delta_{r}(i+1)\cdot\bw_{l}(i)^\top\bw_{r}(i+1)]
\end{align*}

Case 1: $r\neq l$. In this case, $\delta_l$ are independent with $\delta_r$. So we have
\begin{align*}
    &\mathbb{E}[\delta_{l}(i)\delta_{r}(i+1)\cdot\bw_{l}(i)^\top\bw_{r}(i+1)] \\[2mm]
    &= \mathbb{E}[\delta_{l}(i)\delta_{r}(i+1)\cdot\bw_{l}(i)^\top(\bw_{r}(i)-\eta\cdot a_{r}(i)\cdot(\btheta(i)-\btheta^*)^\top] \\[2mm]
    &= \mathbb{E}[\delta_{l}(i)\delta_{r}(i+1)\cdot\bw_{l}(i)^\top\bw_{r}(i)] -\mathbb{E}[\delta_{l}(i)\delta_{r}(i+1)\cdot\bw_{l}(i)^\top\eta\cdot a_{r}(i)\cdot(\btheta(i)-\btheta^*)^\top] \\[2mm]
    &= \mathbb{E}[\delta_{l}(i)]\cdot\mathbb{E}[\delta_{r}(i+1)\cdot\bw_{l}(i)^\top\bw_{r}(i)] - \mathbb{E}[\delta_{l}(i)]\cdot\mathbb{E}[\delta_{r}(i+1)\cdot\bw_{l}(i)^\top\eta\cdot a_{r}(i)\cdot(\btheta(i)-\btheta^*)^\top] \\[2mm]
    &= 0
\end{align*}

Case 2: $r = l$, by Lemma \ref{Markov Stationary}, we have
\begin{align*}
    &\mathbb{E}[\delta_{l}(i)\delta_{l}(i+1)\cdot\bw_{l}(i)^\top\bw_{l}(i+1)] \\[2mm]
    &= \mathbb{E}[\delta_{l}(i)\delta_{l}(i+1)\cdot\bw_{l}(i)^\top(\bw_{l}(i)-\eta\cdot a_{l}(i)\cdot(\btheta(i)-\btheta^*)^\top)] \\[2mm]
    &= \mathbb{E}[\delta_{l}(i)\delta_{l}(i+1)\cdot\bw_{l}(i)^\top\bw_{l}(i)] -\mathbb{E}[\delta_{l}(i)\delta_{l}(i+1)\cdot\bw_{l}(i)^\top\eta\cdot a_{l}(i)\cdot(\btheta(i)-\btheta^*)^\top] \\[2mm]
    &= \mathbb{E}[\delta_{l}(i)\delta_{l}(i+1)]\cdot\mathbb{E}[\Vert\bw_{l}(i)\Vert^2] -\mathbb{E}[\delta_{l}(i)\delta_{l}(i+1)\cdot\bw_l(i)^\top\eta\cdot a_{l}(i)\cdot(\btheta(i)-\btheta^*)^\top] \\[2mm]
    &= -\dfrac{\eta^{0.5}}{2}\mathbb{E}[\Vert\bw_{l}(i)\Vert^2] - \eta\mathbb{E}[\delta_{l}(i)\delta_{l}(i+1)\cdot a_{l}(i)\cdot(\btheta(i)-\btheta^*)^\top\cdot\bw_l(i)] \\[2mm]
    &= -\dfrac{\eta^{0.5}}{2}\mathbb{E}[\Vert\bw_{l}(i)\Vert^2] - O(\eta^{1.75})
\end{align*}

Thus we have 

\begin{align*}
\mathbb{E}[(\bm{\delta}(i)\bW(i))\cdot(\bm{\delta}(i+1) \bW(i+1))^\top] &= \sum_{1\leq l,r\leq m} \mathbb{E}[\delta_{l}(i)\delta_{r}(i+1)\cdot\bw_{l}(i)^\top\bw_{r}(i+1)]\\
    &= \sum_{l=1}^m \mathbb{E}[\delta_{l}(i)\delta_{l}(i+1)]\cdot\bw_{l}(i)^\top\bw_{l}(i+1))\\ 
    &= -\dfrac{\eta^{0.5}}{2}\sum_{l=1}^m \mathbb{E}[\Vert\bw_{l}(i)\Vert^2] - O(m\cdot \eta^{1.75})\\
    &= -\dfrac{\eta^{0.5}}{2}\mathbb{E}[\Vert \bW(i)\Vert^2] - O(m\cdot \eta^{1.75})
\end{align*}

Then we retain

\begin{align} 
    \sum_{i=0}^{t-1}\sum_{j<i}\mathbb{E}[(\bm{\delta}(i)\bW(i))\cdot(\bm{\delta}(j) \bW(j))^\top] &= \sum_{i=0}^{t-2}\mathbb{E}[(\bm{\delta}(i)\bW(i))\cdot(\bm{\delta}(i+1) \bW_{i+1})^\top] \\[2mm]
    &= -\dfrac{\eta^{0.5}}{2}\cdot\sum_{i=0}^{t-2}\mathbb{E}[\Vert \bW(i)\Vert^2] - O(m\cdot\eta^{1.75})\label{deltaW}
\end{align}

Combining equ (\ref{theta^2}) and  equ (\ref{deltaW}), we have

\begin{align*}
    \mathbb{E}[\Vert(\btheta(t)-\btheta^*)\Vert^2] &=\mathbb{E}[\Vert(\btheta(0)-\btheta^*)\Vert^2] \\
    &\quad + \eta^{0.5}\cdot\sum_{i=0}^{t-1}\mathbb{E}[\Vert \bW(i)\Vert^2] + 2\sum_{i=0}^{t-1}\sum_{j<i}\mathbb{E}[(\bm{\delta}(i)\bW(i))\cdot(\bm{\delta}(j) \bW(j))^\top] +O(m\cdot\eta^{1.75})\\
    &= \mathbb{E}[\Vert(\btheta(0)-\btheta^*)\Vert^2] + \eta^{0.5}\cdot\sum_{i=0}^{t-1}\mathbb{E}[\Vert \bW(i)\Vert^2] - \eta^{0.5}\cdot\sum_{i=0}^{t-2}\mathbb{E}[\Vert \bW(i)\Vert^2]+O(m\cdot\eta^{1.75}) \\
    &= \mathbb{E}[\Vert(\btheta(0)-\btheta^*)\Vert^2] + \eta^{0.5}\mathbb{E}[\Vert \bW(t-1)\Vert^2] - O(m\cdot\eta^{1.75})\\[2mm]
    &\geq \mathbb{E}[\|\btheta^\star\|^2] + \eta^{0.5}\mathbb{E}[\Vert \bW(t-1)\Vert^2] 
\end{align*}

Then we give a proof of this theorem by contradiction. Assume for every step $t\leq \frac{1}{\eta^2}$, we have 
\begin{equation}
\mathbb{E}[\dfrac{1}{m}\sum_{i=1}^m \Vert \bm{w_{i}}(T)\Vert^2]\geq \sqrt{\eta}
\end{equation}
i.e. 
\begin{equation}
\mathbb{E}[\|\bW(t)\|^2]\geq m\cdot\sqrt{\eta}
\end{equation}

When $T = \dfrac{1}{\eta^{2}}$, according to equ (\ref{relation}), sum up from $0$ to $T-1$ and we have 
\begin{align*}
 \mathbb{E}[\dfrac{1}{m}\sum_{i=1}^m \Vert \bm{w_{i}}(T)\Vert^2] &= \mathbb{E}[\dfrac{1}{m}\sum_{i=1}^m \Vert \bw_{i}(0)\Vert^2] - (\dfrac{2\eta}{m} - \eta^{2.5})\cdot\sum_{t=0}^{T-1}\mathbb{E}[\Vert(\btheta(t)-\btheta^*)\Vert^2] - \dfrac{2\eta}{m}\cdot\sum_{t=0}^{T-1}\mathbb{E}[(\btheta(t)-\btheta^*)^\top\cdot \btheta^*]\\ 
 &\leq \mathbb{E}[\dfrac{1}{m}\sum_{i=1}^m \Vert \bw_{i}(0)\Vert^2] - (\dfrac{2\eta}{m} - \eta^{2.5})\cdot(T\cdot\mathbb{E}[\Vert\btheta^*\Vert^2] + \eta^{0.5}\sum_{t=0}^{T-1}\mathbb{E}[\Vert \bW_{t}\Vert^2]) + \dfrac{T\cdot2\eta}{m}\cdot \mathbb{E}[\|\btheta^*\|^2] \\
 &\leq \mathbb{E}[\dfrac{1}{m}\sum_{i=1}^m \Vert \bw_{i}(0)\Vert^2] - 2\eta^{1.5}\cdot T\cdot\sqrt{\eta}+  O(T\cdot\eta^{2.5}) \\
 &= 1 - 2\eta^{1.5}\cdot \dfrac{1}{\eta^2}\cdot\sqrt{\eta}+  O(T\cdot\eta^{2.5}) = -1 + O(T\cdot\eta^{2.5}) < 0
\end{align*}
Which is absolutely a contradiction! Therefore, we complete the proof. \quad\(\square\) 

\section{Phase II: Feature Learning and Convergence}
\subsection{Rotation to alignment}
We use gradient descent in phase II. For every neuron $\bw_i$ and $a_i$ ($i\in[m]$) at step $t$:

\begin{gather}
    \dfrac{\partial \cL(\btheta(t))}{\partial \bw_{i}(t)^\top} = a_{i}(t)\cdot \dfrac{1}{n} \sum_{j=1}^m(\ba(t)^\top\bW(t)\bx_j - \by_j)\cdot\bx_j^\top \\
     \dfrac{\partial \cL(\btheta(t))}{\partial a_{i}(t)} =  \dfrac{1}{n} \sum_{j=1}^m(\ba(t)^\top\bW(t)\bx_j - \by_j)\cdot \bx_j^\top \cdot \bw_{i}(t)^\top 
\end{gather}

The update rule of gradient descent is
\begin{gather}
    \bw_{i}(t+1) = \bw_{i}(t) - \eta\cdot a_{i}(t)\cdot \dfrac{1}{n} \sum_{j=1}^m(\ba(t)^\top\bW(t)\bx_j - \by_j)\bx_j \label{w_i update GD}\\[2mm] 
    a_{i}(t+1) = a_{i}(t) - \eta\cdot\dfrac{1}{n} \sum_{j=1}^m(\ba(t)^\top\bW(t)\bx_j - \by_j)\bx_{j}^\top \cdot\bw_{i}(t) \label{a_i update GD}
\end{gather}

By law of large numbers, we have
\begin{equation}
    \dfrac{1}{n}\sum_{j=1}^n \bx_j\cdot\bx_j^\top = \mathbb{E}[\bx_j\cdot\bx_j^\top] + O(\dfrac{1}{\sqrt{n}}) = I + O(\dfrac{1}{\sqrt{n}})
\end{equation}
Then we have
\begin{align}
    \bw_{i}(t+1) &= \bw_{i}(t) - \eta\cdot a_{i}(t)\cdot \dfrac{1}{n} \sum_{j=1}^m \bx_j\cdot \bx_j^\top\cdot(\btheta(t) - \btheta^*) \\
    &= \bw_{i}(t) - \eta\cdot a_{i}(t)\cdot(\btheta(t) - \btheta^*) + O(\dfrac{\eta}{\sqrt{n}})  
\end{align}
    
And 
\begin{align*}
    a_{i}(t+1) &= a_{i}(t) - \eta\cdot(\btheta(t) - \btheta^*)\dfrac{1}{n} \sum_{j=1}^m \bx_j \cdot \bx_{j}^\top \cdot\bw_{i}(t) \\
    &=  a_{i}(t) - \eta\cdot(\btheta(t) - \btheta^*)\cdot\bw_{i}(t)+O(\dfrac{\eta}{\sqrt{n}})  
\end{align*}

Since $a_i(t),\|\bw_i(t)\|\leq n^{0.5}$ when phase II begins,  we have
\begin{align*}
     \bw_{i}(t+1) &= \bw_{i}(t) - \eta\cdot a_{i}(t)\cdot(\btheta(t) - \btheta^*) + O(\dfrac{\eta}{\sqrt{n}})   \\[2mm]
     & = \bw_{i}(t) + \eta\cdot (a_{i}(t)\cdot\btheta^*) + O(\eta^{1.5})
\end{align*}

Similarly, we have:
\begin{align*}
     a_{i}(t+1) 
     &=a_{i}(t)+ \eta\cdot\btheta^*\cdot\bw_{i}(t) + O(\eta^{1.5})
\end{align*}

Combine both we have:

\begin{equation}
    \left[\begin{matrix}
        \bw_{i}(t+1) \\[2mm]
        a_{i}(t+1)
    \end{matrix}\right] = (\bm{I} + \eta\bm{M})\cdot \left[\begin{matrix}
        \bw_{i}(t)\\[2mm]
        a_{i}(t)
    \end{matrix}\right] +O(\eta^{1.5}) \label{phaseIIiter}
\end{equation}
where
$
\bm{M}=
\left[
    \begin{matrix}
    \bm{0} & {\btheta^*}^\top \\[2mm]
    \btheta^* & 0
    \end{matrix}
    \right] 
$.
The top eigenvalue of $\bm{M}$ is $\lambda_1 = \Vert \btheta^*\Vert$ and the lowest eigenvalue of $\bm{M}$ is $\lambda_{n+1} = -\Vert \btheta^*\Vert$. All the
other eigenvalues of $\bm{M}$ are equal to 0. Since $\bm{M}$ is symmetry matrix, there exists orthogonal matrix $\bm{Q_M}$ such that \begin{equation}
   \bM =  \bQ_M^\top\cdot\text{diag}(\Vert \btheta^*\Vert, 0, \dots, -\Vert \btheta^*\Vert)\cdot \bQ_M 
\end{equation}

\begin{lemma}[Alignment]\label{phase II alignment}
     Suppose \cref{cond:main}~(A1-3, 5-6) holds and consider gradient descent for updates.
     Assume that phase II begins at time $t_1$, then at time $t_2=t_1+T_2$, $T_2=\frac{1}{\|\btheta^{\star}\|}\cdot \ln(\frac{1}{\eta})$, for any neuron $\bw_{i}$ it holds
    \begin{equation}
        \dfrac{\left\lvert \langle\boldsymbol{\theta}^{\star},\boldsymbol{w}_{i}(t_2)\rangle\right\lvert}{\|\boldsymbol{\theta}^{\star}\|\cdot\|\boldsymbol{w}_{i}(t_2)\|}\geq  1- \left\lvert O(\ln\frac{1}{\eta}\cdot\sqrt{\eta})\right\lvert.
    \end{equation} 
\end{lemma}

% if $|\langle(\bm{A^*})^\top,\bm{w_{1l}}(t_0)\rangle| \geq \dfrac{1}{2}$ (\textcolor{red}{has a good alignment}) and $\langle(\bm{A^*})^\top,\bm{w_{1l}(t_0)}\rangle\cdot\langle\mathbb{E}[w_{2l}(t)](\bm{A^*})^\top,\bm{w_{1l}(t_0)}\rangle < 0$, and $\Vert \bm{w_{1l}(t_0)}\Vert \geq c(\eta)$( \textcolor{red}{still active} ) then 

\textbf{Proof}. By equ (\ref{phaseIIiter}), we have
\begin{align*}
  \left[\begin{matrix}
        \bw_{i}(t_1+T_2) \\[2mm]
        a_{i}(t_1+T_2)
    \end{matrix}\right]&= (\bm{I} +\eta\cdot\bm{M})^{T_2}\cdot \left[\begin{matrix}
        \bw_{i}(t_1) \\[2mm]
        a_{i}(t_1)
    \end{matrix}\right] + O(T_2\cdot \eta^{1.5}) \\[2mm]
    &= \sum_{k=0}^{T_2} \binom{T_2}{k} (\eta \bm{M})^k\cdot \left[\begin{matrix}
        \bw_{i}(t_1) \\[2mm]
        a_{i}(t_1)
    \end{matrix}\right] + O(T_2\cdot \eta^{1.5})  \\[2mm]
    &= (I + T_2\eta\bm{M} + \frac{T_2(T_2-1)}{2}(\eta \bm{M})^2 + \cdots + (\eta M)^T)\cdot \left[\begin{matrix}
        \bw_{i}(t_1) \\[2mm]
        a_{i}(t_1)
    \end{matrix}\right] + O(T_2\cdot \eta^{1.5}) \\[2mm]
    &= (\sum_{i=0}^{T_2}\frac{(T_2\eta\bm{M})^i}{i!} + O(T_2\cdot\eta^2))\cdot \left[\begin{matrix}
        \bw_{i}(t_1) \\[2mm]
        a_{i}(t_1)
    \end{matrix}\right] + O(T_2\cdot \eta^{1.5})  \\[2mm]
    &= (\exp(T_2\eta\cdot \bM) + O(\frac{(T_2\eta\cdot \bM)^{T+1}}{(T+1)!}))\cdot \left[\begin{matrix}
        \bw_{i}(t_1) \\[2mm]
        a_{i}(t_1)
    \end{matrix}\right] + O(T_2\cdot \eta^{1.5}) \quad(\text{taylor expansion}) \\[2mm] 
     &= \dfrac{e\cdot\bM}{\|\btheta^*\|}\left[\begin{matrix}
        \bw_{i}(t_1) \\[2mm]
        a_{i}(t_1)
    \end{matrix}\right] + O(T_2\cdot \eta^{1.5}) \quad(\text{$\eta^{-\ln\eta} \ll (-\ln\eta)^{-\ln\eta}$}) \\[2mm] 
\end{align*}

So we have 
\begin{equation}
     \bw_{i}(t_1+T_2) =  \dfrac{e\cdot a_{i}(t_1)}{\Vert \btheta^*\Vert
    _2}\cdot {\btheta^*}^\top + O(\ln \dfrac{1}{\eta}\cdot \eta^{1.5})  
\end{equation}

Then we have, 
\begin{align*}
     \left\lvert \langle\btheta^*,\bw_{i}(t_2)\rangle\right\lvert &= \left\lvert\langle\btheta^*,\dfrac{e\cdot a_{i}(t_1)}{\Vert \btheta^*\Vert
    _2}\cdot {\btheta^*}^\top +O(\ln \dfrac{1}{\eta}\cdot \eta^{1.5})\rangle\right\lvert    \\[2mm]
    &= \dfrac{\left\lvert\Vert\btheta^*\Vert\cdot e\cdot a_i(t_1)+\btheta^*\cdot O(\ln \dfrac{1}{\eta}\cdot \eta^{1.5})\right\lvert}{\Vert\btheta^*\Vert\cdot\Vert\bw_{i}(t_2)\Vert} \\[2mm]
    &=\dfrac{\left\lvert e\cdot a_i(t_1)+\dfrac{\btheta^*}{\Vert\btheta^*\Vert}\cdot O(\ln \dfrac{1}{\eta}\cdot \eta^{1.5})\right\lvert}{\left\Vert\dfrac{e\cdot a_{i}(t_1)}{\Vert \btheta^*\Vert}\cdot {\btheta^*}^\top + O(\ln \dfrac{1}{\eta}\cdot \eta^{1.5})  \right\Vert} \\[2mm]
     &\geq \dfrac{e\cdot a_i(t_1) - \left\lvert O(\ln\dfrac{1}{\eta}\cdot\eta^{1.5})\right\lvert}{e\cdot a_i(t_1) + \left\lvert O(\ln\dfrac{1}{\eta}\cdot\eta^{1.5}))\right\lvert}\quad(\text{triangle inequality})\\[2mm]
     &= 1-\dfrac{\left\lvert O(\ln\frac{1}{\eta}\cdot\eta^{1.5})\right\lvert}{e\cdot a_i(t_1) + \left\lvert O(\ln\frac{1}{\eta}\cdot\eta^{1.5})\right\lvert} \geq 1 - \left\lvert O(\ln\frac{1}{\eta}\cdot\sqrt{\eta})\right\lvert 
\end{align*}
Therefore, we complete the proof.\quad \(\square\)

\subsection{Convergence to sparse solution}

We assume that all the neurons are perfect aligned, i.e.
for every neuron $\bw_i$ ($i\in[m]$) at step $t$, there exists coefficient $\gamma_i(t)$ such that $\bw_i(t) = \gamma_i(t)\cdot \btheta^*$. 

\begin{lemma}[Convergence]\label{phase II norm increase}
    Suppose \cref{cond:main}~(A1-3, 5-6) holds and consider gradient descent for updates.
    Assume all the neurons are perfectly aligned at step $t_2$. Let $t_3 = t_2 + \frac{1}{\|\btheta^{\star}\|^2}\cdot\frac{\ln(1/\eta)}{\eta} $. Using gradient descent, we have $\|\btheta(t_3) - \btheta^{\star}\|\leq |O(\eta\cdot\ln\frac{1}{\eta})|$. Furthermore, for any neuron $\|\bw_i(t_3)\|\geq \sqrt{\eta}$ ($i\in[m]$), we have 
    \begin{equation}
\dfrac{\left\lvert\langle\boldsymbol{\theta}^{\star},\boldsymbol{w}_i(t_3)\rangle\right\lvert}{\|\boldsymbol{\theta}^{\star}\|\cdot\|\boldsymbol{w}_{i}(t_3)\|}\geq 1 - \left\lvert O(\eta\cdot\ln\frac{1}{\eta})\right\lvert.
    \end{equation}
\end{lemma}
By equ (\ref{w_i update GD}), when $t=t_2 $, we have
\begin{align}
    \bw_{i}(t+1) &= \bw_{i}(t) - \eta\cdot a_{i}(t)\cdot \dfrac{1}{n} \sum_{j=1}^m(\ba(t)^\top\bW(t) \bx_j - \by_j)\bx_j \\
    &= \gamma_i(t)\cdot \btheta^* - \eta\cdot a_{i}(t)\cdot(\dfrac{1}{n}\sum_{j=1}^m \bx_j\cdot\bx_j^\top)\cdot(\sum_{k=1}^m a_k(t)\cdot\gamma_k(t)\cdot\btheta^*-\btheta^*)  \\
    &= \left\{\gamma_i(t) -\eta\cdot a_{i}(t)\cdot (\sum_{k=1}^m a_k(t)\cdot\gamma_k(t)-1)\right\}\cdot \btheta^* + O(\dfrac{\eta}{\sqrt{n}}) 
\end{align}
The last equation is due to $\text{$\dfrac{1}{n}\sum_{j=1}^m \bx_j\cdot\bx_j^\top = I + O(\frac{1}{\sqrt{n}}) = I + O(\eta^2)$ by large number law} $ and \cref{cond:main}(A3). Notice that
\begin{equation}
    \btheta(t) = \ba(t)^\top\cdot\bW(t) = \sum_{i=1}^m a_i(t)\cdot\bw_i(t)= (\sum_{i=1}^m a_i(t)\cdot\gamma_i(t))\cdot\btheta^* 
\end{equation}

Then we have 
\begin{align*}
    \btheta(t+1) &= \btheta(t) - \eta\cdot \ba(t)^\top\ba(t) \cdot \dfrac{1}{n}\sum_{i=1}^n(\ba(t)^\top \bW(t)\bx_i - y_i)\cdot \bx_i^\top \\
    &\quad - \eta\cdot \dfrac{1}{n}\sum_{i=1}^n(\ba^\top(t)\bW(t)\bx_i - y_i)\cdot \bx_i^\top \bW(t)^\top \bW(t) + O(\eta^2) \\
    &= \btheta(t) - \eta\cdot\sum_{i=1}^m a_i(t)^2\cdot (\sum_{j=1}^m a_i(t)\cdot\gamma_i(t)-1)\cdot\btheta^* - \eta\cdot(\sum_{j=1}^m a_i(t)\cdot\gamma_i(t)-1)\cdot\|\btheta^*\|^2\cdot\btheta^*+O(\eta^2) \\
    &= \btheta(t) -\eta\cdot(\sum_{j=1}^m a_i(t)\cdot\gamma_i(t)-1)\cdot(\sum_{i=1}^m a_i(t)^2 + \|\btheta^*\|^2)\cdot\btheta^* + O(\eta^2)
\end{align*}
We consider the change of $\btheta(t)-\btheta^*$.
Subtracting $\theta^*$ on both side and we retain
\begin{align}
    \btheta(t+1) - \btheta^* &= \btheta(t) -\btheta^*-\eta\cdot(\sum_{j=1}^m a_i(t)\cdot\gamma_i(t)-1)\cdot(\sum_{i=1}^m a_i(t)^2 + \|\btheta^*\|^2)\cdot\btheta^* + O(\eta^2) \\
    &= (\sum_{k=1}^m a_k(t)\cdot\gamma_k(t)-1)\cdot(1 - \eta\cdot(\sum_{i=1}^m a_i(t)^2 + \|\btheta^*\|^2))\cdot\btheta^* + O(\eta^2) \\
    &= (\btheta(t)-\btheta^*)\cdot(1 - \eta\cdot(\sum_{i=1}^m a_i(t)^2 + \|\btheta^*\|^2)) + O(\eta^2) \label{perfect iter}
\end{align}

In the beginning, $|a_k(t_2)|,\|\bw_k(t_2)\|\leq \sqrt{\eta}$, so $\gamma_i(t_2)\leq \frac{\sqrt{\eta}}{\|\btheta^*\|}$ and we have
\begin{equation}
    \left\lvert\sum_{k=1}^m a_k(t_2)\cdot\gamma_k(t_2)
    \right\lvert \leq m\cdot\dfrac{\eta}{\|\btheta^*\|} 
\end{equation}

% We claim: there exists a step $t_1 = t_0 + T$ such that 
% \begin{equation}
%     % \left\lvert\sum_{k=1}^m a_k(t +T)\cdot\gamma_k(t+T)\right\lvert\geq 1- \varepsilon \quad or \quad 
%     \sum_{i=1}^m a_i(t_1)^2 \geq \dfrac{1}{\eta} - \|\btheta^*\| - \varepsilon
% \end{equation}
% Otherwise,  
by equ (\ref{perfect iter}), for any $t_2\leq t\leq t_3$ it holds
\begin{align}
   \| \btheta(t+1) - \btheta^*\| &=  \| \btheta(t) - \btheta^*\|\cdot\left\lvert(1 - \eta\cdot(\sum_{i=1}^m a_i(t)^2 + \|\btheta^*\|^2))\right\lvert + O(\eta^2) \\[2mm]
   &\leq \| \btheta(t) - \btheta^*\|\cdot (1-\eta\cdot\|\btheta^*\|^2)+O(\eta^2)
\end{align}

Let $T = \frac{-\ln\eta}{\|\btheta^*\|^2\cdot\eta}$, we retain
\begin{equation}
    \| \btheta(t_3) - \btheta^*\|\leq  \| \btheta(t_2) - \btheta^*\|\cdot(1 - \eta\cdot\|\btheta^*\|^2)^T + O(T\cdot\eta^2)
\end{equation}

Since $(1-x)^T\leq \exp({-x\cdot T})$, we have
\begin{align*}
     \| \btheta(t_3) - \btheta^*\|&\leq  \| \btheta(t_2) - \btheta^*\|\cdot(1 - \eta\cdot\|\btheta^*\|^2)^T + O(T\cdot\eta^2)\\[2mm]
     &\leq  \| \btheta(t_2) - \btheta^*\|\cdot\exp(-\eta\cdot\|\btheta^*\|^2\cdot T) + O(\eta\cdot\ln\frac{1}{\eta}) \\[2mm]
      &=  \| \btheta(t_2) - \btheta^*\|\cdot\eta + O(\eta\cdot\ln\frac{1}{\eta}) \\[2mm]
      &= \left\lvert O(\eta\cdot\ln\frac{1}{\eta})\right\lvert
\end{align*}

So we reatin 
\begin{equation}
   \left\lvert 1 -  \sum_{k=1}^m a_k(t_3)\cdot\gamma_k(t_3)\right\lvert\leq \left\lvert O(\eta\cdot\ln\frac{1}{\eta})\right\lvert \label{t_3bound}
\end{equation}
Since the summation of $\bw_i(i)$ approach to $\btheta^*$, which shows that there exists some neurons with "big" norm. For any neuron $\|\bw_i(t_3)\|\geq \sqrt{\eta}$, we have  $\gamma_i(t_3)\geq \frac{\sqrt{\eta}}{\|\btheta^*\|}$. For these neuron, we retain
\begin{align*}
\left\lvert\langle\bw_i(t_3),\btheta^*\rangle\right\lvert &= \left\lvert\langle (\sum_{k=1}^m a_k(t_3)\cdot\gamma_k(t_3))\cdot \btheta^* + O(T\cdot\eta^2) , \btheta^*\rangle\right\lvert\\[2mm]
&= \left\lvert\dfrac{ (\sum_{k=1}^m a_k(t_3)\cdot\gamma_k(t_3))\cdot\|\btheta^*\|^2 + O(\eta\cdot\ln\dfrac{1}{\eta})}{{\gamma_i(t_3)\cdot\|\btheta^*\|^2 + O(\eta\cdot\ln\frac{1}{\eta})}}\right\lvert \\[2mm]
&\geq \dfrac{\left\lvert\gamma_i(t_3)\cdot\Vert\btheta^*\Vert^2\right\lvert - \left\lvert O(\eta\cdot\ln\dfrac{1}{\eta})\right\lvert }{\left\lvert \gamma_i(t_3)\cdot\|\btheta^*\|^2\right\lvert  + \left\lvert O(\eta\cdot\ln\frac{1}{\eta})\right\lvert }\quad\text{(by triangle inequality and equ (\ref{t_3bound}) )}\\[2mm]
&= 1 - \left\lvert O(\eta\cdot\ln\dfrac{1}{\eta})\right\lvert
\end{align*}
Therefore, we complete the proof. \quad\(\square\)

\section{Experiments details}
\textbf{Main Experimental Setup.}
We perform our empirical observations on commonly used image classification datasets CIFAR-10 \citet{krizhevsky2009learning}, with the standard architectures ResNet-18 \citet{kaiming2016residual}.
Optimization is done with vanilla and label noise SGD, with no weight decay or momentum.
The learning rate is set to $0.1$, and the total number of epochs is $160$.

\textbf{Model Pruning and Sparsity.}
% Modern deep neural networks are often dense and computationally expensive. 
Model pruning is a compression technique that reduces network size by selectively removing redundant components, such as weights and neurons, while maintaining comparable performance. 
The pruning limit of a trained neural network is often linked to its sparsity~\citep{diao2023pruning}, as weights and neurons with small magnitudes are typically considered redundant or non-influential~\citep{Hagiwara1993removal,han2015deep,frankle2018lottery}.
Therefore, a sparser model tends to retain higher performance after pruning to the same level.
In this work, we use model pruning to investigate the sparsity of solutions obtained through different training algorithms.
Specifically, we focus on the Iterative Feature Merging (IFM) algorithm~\citep{chen2024going}. 
To the best of our knowledge, IFM is the first pruning technique that does not require fine-tuning or modifications to the training procedure and thus can be directly applied to a pre-trained model.

\section{Additional experiment}

\begin{figure*}[tb!]
    \begin{center}
\includegraphics[width=0.75\textwidth]{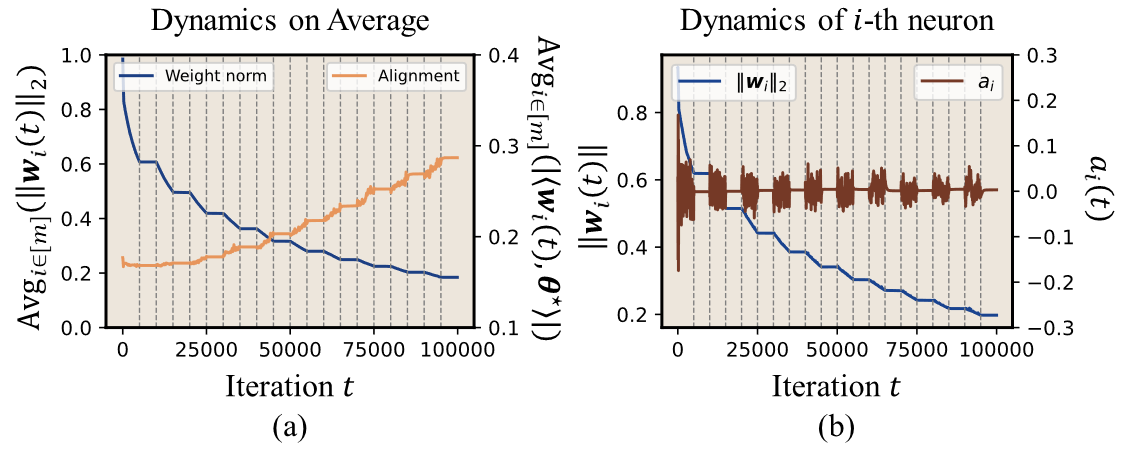}
        \caption{ \textbf{Neuron dynamics in alternating noisy/noiseless SGD training.}
        We replicate the synthetic problem setup from Subsection "Setup and Overview: A Two-Layer Linear Network", changing the standard SGD label noise to training with alternating noisy and noiseless SGD every 5000 steps.
         \textbf{(a) Learning dynamics on average.}
        The averaged neuron norm $\text{Avg}_{i\in [m]}(\Vert \bw_i (t) \Vert_2)$ and the averaged neuron alignment $\text{Avg}_{i\in [m]}(\langle \bw_i(t), \btheta^{\star} \rangle)$ vs. training iteration $t$.
        \textbf{(b) Learning dynamics of $i$-th neuron.}
        The $i$-th neuron norm $\Vert \bw_i (t) \Vert_2$ and the corresponding absolute value $|a_i(t)|$ in the second layer vs. training iteration $t$.
            \label{fig:alternating nosie}
        }
    \end{center}
\end{figure*}

In this section, we present the supplementary experiment to help understand the mechanistic principles of label noise SGD.

\noindent
\textbf{Alternating label noise SGD:}
We trained the model using SGD but alternate label noise every 5,000 steps (adding label noise for 5000 steps and then remove label noise for 5000 steps). As shown in \cref{fig:alternating nosie}, this experiment reveals two key phenomena:  (i) the norms of first-layer neurons gradually decrease when adding label noise and the norm reduction stops when removing label noise. (ii) the second-layer neuron weights oscillate near zero under noise but stabilize.
These observations collectively demonstrate that the crucial role  of label noise in facilitating the transition from lazy to rich regimes. Our theoretical analysis further in paper establishes that label noise SGD induces the second-layer oscillations, driving the progressive reduction of first-layer norms, thereby enabling the model to escape the lazy regime.

% \newpage

\end{document}